\documentclass[letterpaper]{article} % DO NOT CHANGE THIS
\usepackage[preprint]{aaai2027}  % DO NOT CHANGE THIS
\usepackage[hyphens]{url}  % DO NOT CHANGE THIS
\usepackage{graphicx} % DO NOT CHANGE THIS
\usepackage{natbib}  % DO NOT CHANGE THIS AND DO NOT ADD ANY OPTIONS TO IT
\usepackage{caption} % DO NOT CHANGE THIS AND DO NOT ADD ANY OPTIONS TO IT
\usepackage{algorithm}
\usepackage{algorithmic}
\usepackage{newfloat}
\usepackage{listings}
\DeclareCaptionStyle{ruled}{labelfont=normalfont,labelsep=colon,strut=off} % DO NOT CHANGE THIS
\floatstyle{ruled}
\newfloat{listing}{tb}{lst}{}
\floatname{listing}{Listing}
\usepackage{amsmath,amssymb,amsfonts,amsthm}
\usepackage{booktabs}
\usepackage{bm}
\newtheorem{proposition}{Proposition}
\newtheorem{theorem}{Theorem}
\newtheorem{lemma}{Lemma}
\newtheorem{corollary}{Corollary}

\newcommand{\ba}{\bm a}

\newcommand{\A}{\mathcal A}
\newcommand{\F}{\mathcal F}
\newcommand{\E}{\mathbb E}
\newcommand{\R}{\mathbb R}
\newcommand{\HH}{\mathbf H}
\newcommand{\one}{\mathbf 1}

\title{Characterizing Oblivious Common-Reward Multi-Player Bandits: Universal Stationarity, Coordination Traps, and Certified Global Rates}
\author{Mingxiao Li, Wenhai Lai, and Hongyu Wang}
\affiliations{}

\begin{document}

\maketitle

\begin{abstract}
This paper presents a theoretical analysis of the \emph{oblivious common-reward multi-player multi-armed bandit} (OC-MPMAB), a minimal multi-player bandit model in which multiple players select arms simultaneously and observe only a shared reward. Crucially, the players are fully oblivious: they do not observe the actions of other players and do not even know whether any other players exist. The contributions of this work are threefold. First, we prove that the independent gradient-bandit updates applied by the players jointly form an unbiased stochastic gradient estimator of a hidden product-policy objective. This property then yields a universal finite-time first-order stationarity guarantee for every instance. Second, we show that stationarity cannot, in general, be strengthened to global optimality. To illustrate this, we identify a product-policy trap obstruction that rules out the guarantee of global optimality. Third, we show that global optimality can be recovered under structural assumptions on the reward table. To make global optimality achievable, we introduce finite-table certificates that separately characterize product-gradient dominance and self-bounding stochastic noise. We then establish certificate-conditioned bounds on the product-policy gap and pseudo-regret. Consequently, the standard $O(1/t)$ product-policy gap and $O(\log T)$ pseudo-regret rates hold when the certificates are uniform.
\end{abstract}

\section{Introduction}

Multi-armed bandits provide a fundamental framework for sequential decision making under partial feedback. Classical algorithms, such as UCB and Thompson sampling, achieve strong regret guarantees for single-player stochastic bandits \cite{thompson1933likelihood,lai1985asymptotically,auer2002finite,agrawal2012analysis,lattimore2020bandit}. Recently, stochastic gradient methods provide an alternative approach. For example, likelihood-ratio estimators, REINFORCE, and gradient bandit algorithms update softmax policies using sampled actions and rewards \cite{robbins1951stochastic,glynn1990likelihood,williams1992simple,sutton2018reinforcement}. For single-player bandits, these methods can achieve global optimality under self-bounding noise and weak exploration conditions \cite{mei2023stochastic}.

These single-player guarantees rely on the assumption that each arm has a fixed reward distribution. This assumption may fail in modular AI systems, where separate controllers can adapt prompts, retrieval, tool usage, and verification based on the same end-to-end score. Due to black-box interfaces, access restrictions, or separate ownership, the decisions and even the existence of other adaptive components may remain hidden. Consequently, a controller may appear to face a standard bandit problem, while the value of each local arm changes as the hidden controllers adapt. This hidden coupling motivates a minimal model that characterizes what can be achieved from shared rewards alone. We formalize this setting as the \emph{oblivious common-reward multi-player multi-armed bandit} (OC-MPMAB). In an OC-MPMAB, each player selects a local arm, the environment evaluates the joint action, and all players receive the same reward. Crucially, each player observes only its selected arm and the shared reward. In the setting studied here, each player applies the standard gradient-bandit update independently. From a centralized perspective, however, the independent softmax policies induce a product distribution over joint actions. As a main contribution of this work, we show that the local updates jointly constitute an unbiased stochastic gradient estimator of the expected-reward objective induced by this hidden product policy.

Moreover, the model considered in this work differs fundamentally from both equilibrium selection in common-interest games and payoff-based learning under a specified game model. Existing analyses typically assume a known game form, observable joint actions, or prescribed revision dynamics. In contrast, players in OC-MPMAB have access to none of this information. Therefore, the theoretical convergence results of these models cannot be directly applied to OC-MPMAB. In this paper, we revisit and re-establish the theoretical convergence guarantees of OC-MPMAB. Specifically, we show that finite-time stationarity is guaranteed for every finite normalized reward table, whereas global optimality cannot be guaranteed in general because the product-policy landscape may contain attracting suboptimal traps. Nevertheless, global bounds on the product-policy gap and pseudo-regret can be recovered when offline certificates for the reward table rule out such traps and control stochastic sampling noise. Accordingly, the positive and negative results characterize fundamental properties of the hidden product-policy landscape.

In summary, this paper provides a comprehensive theoretical analysis of the OC-MPMAB model. The contributions of this work are as follows:
\begin{itemize}
\item We prove that the independent local gradient-bandit updates in OC-MPMAB jointly constitute an unbiased stochastic gradient estimator of the centralized expected-reward objective over product policies, even though no player observes either the joint action or this objective. This identity then yields a universal finite-time stationarity guarantee.

\item We show that finite-time stationarity does not imply global optimality for OC-MPMAB: strict trap minors create suboptimal attractors. Separately, on a deterministic block trap, the stochastic update converges from the default initialization to a suboptimal block with constant probability under sufficiently small diminishing stepsizes.

\item We formulate guarantees of global optimality as a finite reward-table certification problem. Specifically, we introduce separate certificates that rule out product-gradient traps and control stochastic sampling noise. These certificates yield bounds on the product-policy gap and pseudo-regret for both sparse coordination reward tables and locally dominant dense reward tables.
\end{itemize}
\textbf{Notation.} The bold lowercase symbol $\ba=(A_1,\ldots,A_N)$ denotes a joint action, where $A_i\in\A_i$ is the local arm of player $i$; $\ba_{-i}$ denotes the profile of all players except $i$, and $\A_{-i}=\prod_{j\ne i}\A_j$. A policy tuple $\pi=(\pi_1,\ldots,\pi_N)$ induces $\pi(\ba)=\prod_i\pi_i(A_i)$, while $\pi_{\HH}$ denotes the product policy parameterized by the preference collection $\HH$. For vectors, $\|\cdot\|_2$ is the Euclidean norm; for matrices, $\|\cdot\|_2$ and $\|\cdot\|_F$ are the spectral and Frobenius norms, respectively.

\section{OC-MPMAB Model}

\subsection{System Model}

An \emph{oblivious common-reward multi-player multi-armed bandit} (OC-MPMAB) consists of a finite player set $ [N]=\{1,2,\ldots,N\} $, where each player $i\in[N]$ is associated with a finite arm set $\A_i$ of cardinality $K_i$. At each round $t$, all players simultaneously select local arms. Denote by $A_{i,t}$ the local arm selected by player $i$. These local arms then form the joint action
\begin{equation}
\ba_t=(A_{1,t},\ldots,A_{N,t})\in\A,
\;
\A=\A_1\times\cdots\times\A_N .
\end{equation}
For every joint action $\ba\in\A$, the environment specifies a stochastic reward distribution. Let $\F_t$ denote the centralized history before arm selection at round $t$. After the joint action $\ba_t$ is selected, the environment generates a scalar reward $R_t$ according to the reward distribution associated with $\ba_t$. Following the standard convention for stochastic bandits, we assume that rewards are normalized:
\begin{equation}
R_t\in[0,1]\quad\text{for all }t.
\end{equation}
We assume this normalized feedback model throughout the paper.
Moreover, we define the mean reward of the joint action $\ba_t$ as
\begin{equation}
r(\ba_t)=\E[R_t\mid \F_t,\ba_t]\in[0,1].
\end{equation}
Therefore, the conditional mean reward depends only on the current joint action through a fixed finite reward table and is independent of the previous history.

In OC-MPMABs, \emph{oblivious} means that each player is unaware of both the existence and the actions of the other players. Specifically, player $i$ observes only the sampled local arm and the reward received in each round. Consequently, each player behaves as a standard single-agent player and executes a conventional gradient bandit algorithm. More importantly, the reward is common: all players observe the same scalar reward $R_t$ at round $t$, and there are no player-specific rewards. We remark that the joint action space $\A$, the common reward function $r$, and the product-policy structure introduced below are used solely for centralized analysis. None of this information is available to any individual player. The next subsection introduces the local gradient bandit update for each player.

% This model differs from classical multi-player bandits with collisions or player-specific rewards, as well as from cooperative bandit optimization, in which multiple private reward functions are aggregated. 

\subsection{Gradient Bandit Algorithm for Each Player}

Because a player has no information about other players, the player applies the standard gradient bandit update using only the sampled local arm and the common reward. Specifically, player $i$ maintains a numerical preference variable $H_{i,t}(x)\in\R$ for every local arm $x\in\A_i$. A larger value of $H_{i,t}(x)$ assigns a higher sampling probability to arm $x$ in round $t$. At the beginning of each round, player $i$ converts these preferences into the softmax distribution
\begin{equation}
\label{eq:softmax}
  \pi_{i,t}(x)=\frac{\exp(H_{i,t}(x))}{\sum_{y\in\A_i}\exp(H_{i,t}(y))},
  \; x\in\A_i .
\end{equation}
The player then samples a local arm according to
\begin{equation}
  A_{i,t}=x\;\text{with probability } \pi_{i,t}(x).
\end{equation}
After all players have chosen their arms, the environment observes the joint action $\ba_t=(A_{1,t},A_{2,t},\ldots,A_{N,t})$ and broadcasts the same reward sample $R_t$ to every player.

It remains to specify the update of each $H_{i,t}(x)$. By convention, every $H_{i,1}(x)$ is initialized to zero. Each player then updates the local preference variables as
\begin{equation}
\label{eq:update}
  H_{i,t+1}(x)=H_{i,t}(x)+\gamma_t g_{i,t}(x),
  \; x\in\A_i,
\end{equation}
where $\gamma_t\ge0$ is a deterministic stepsize and
\begin{align}
\label{eq:define_g_it}
  g_{i,t}(x)&=R_t
  \bigl(\one\{A_{i,t}=x\}-\pi_{i,t}(x)\bigr)\notag\\
  &=
  \begin{cases}
  R_t(1-\pi_{i,t}(x)), & \text{if } A_{i,t}=x,\\
  -R_t\pi_{i,t}(x), & \text{if } x\ne A_{i,t}.
  \end{cases}
\end{align}
Note that the update does not use the number of players, the actions of other players, joint-action samples, messages, or the reward table.

We now examine the above process from a centralized perspective. Let $d=\sum_{i=1}^N K_i$ and let $\HH=(H_i(x):i\in[N],x\in\A_i)\in\R^d$ denote the vector of all preference variables. We suppress the round index and write each $H_i(x)$ as a generic preference variable; the value generated by the update in round $t$ is denoted by $H_{i,t}(x)$. Moreover, we write $\mathbf g_t=(g_{i,t}(x):i\in[N],x\in\A_i)$ for the vector obtained by concatenating all local update coordinates at round $t$. Given $\HH$, the players independently sample their local arms according to the softmax probabilities in \eqref{eq:softmax}. Therefore, the probability assigned to a joint action $\ba=(A_1,A_2,\ldots,A_N)$ is
\begin{equation}
  \label{eq:product_policy}
  \pi_{\HH}(\ba)=\prod_{i=1}^N\pi_{i}(A_{i}).
\end{equation}
Treating $\HH$ as the optimization variable, we define the expected reward as
\begin{equation}
\label{eq:define_J}
  J(\HH)=\E_{\ba\sim\pi_{\HH}}[r(\ba)]
  =\sum_{\ba\in\A}\pi_{\HH}(\ba)r(\ba).
\end{equation}
Although the players update independently, their coupled updates estimate the gradient of $J$. Because finite softmax policies can approach every pure profile, the relevant optimum is
\begin{equation}
    \sup_{\HH\in\R^d}J(\HH)=J^\star,
    \label{eq:actual_problem}
\end{equation}
where $J^\star=\max_{\ba\in\A}r(\ba)$. No player observes this objective or the product policy in \eqref{eq:product_policy}; both are used only for centralized analysis.

\section{Unbiased Product-Policy Gradient}

Unlike in a single-player bandit, the observed reward depends on the full joint action rather than on one local arm. We first show that, despite this coupling, each local update is an unbiased estimator of the corresponding gradient coordinate of $J(\HH)$. The concatenated updates therefore implement stochastic gradient ascent on the objective in \eqref{eq:actual_problem}.

\begin{proposition}
\label{prop:unbiased}
For every round $t\geq 1$, player $i\in[N]$, and arm $x\in\A_i$, the gradient bandit update $g_{i,t}(x)$ in \eqref{eq:define_g_it} is an unbiased estimator of the partial derivative of the expected reward with respect to the preference variable for arm $x$:
\begin{equation}
\E\left[g_{i,t}(x)\mid\F_t\right]
=\left. \frac{\partial J(\HH)}{\partial H_i(x)} \right|_{\HH=\HH_t},
\label{eq:unbiased_coordinate_gradient}
\end{equation}
where the conditional expectation is taken over the sampled arms and the reward in round $t$.
\end{proposition}

\begin{proof}[Proof sketch]
Here and throughout, $\nabla J(\HH)=\left(\partial J(\HH)/\partial H_i(x):i\in[N],x\in\A_i\right)$ denotes the gradient vector of $J$ with respect to the concatenated preference vector $\HH$. Notably,
within the product policy, the only term that depends on $H_i(x)$ is $\pi_i(A_i)$, whose softmax derivative is given by
\begin{equation}
\frac{\partial \pi_i(A_i)}{\partial H_i(x)}
=
\pi_i(A_i)\bigl(\one\{A_i=x\}-\pi_i(x)\bigr).
\end{equation}
Substituting this identity into $\partial J(\HH)/\partial H_i(x)$ and using the identity $\E[R_t\mid\F_t,\ba_t]=r(\ba_t)$ yields the same summation as the conditional expectation of $g_{i,t}(x)$. The complete derivation is provided in the supplementary material.
\end{proof}

\textbf{Effect of baselines.} We remark that adding an action-independent baseline in the update \eqref{eq:define_g_it} will not change the claim in Proposition \ref{prop:unbiased}. Specifically, if $b_t$ is $\F_t$-measurable, its score contribution has zero conditional mean because $\E[\one\{A_{i,t}=x\}-\pi_{i,t}(x)\mid\F_t]=0$; therefore, replacing $R_t$ by $R_t-b_t$ in \eqref{eq:define_g_it} preserves Proposition \ref{prop:unbiased}.

\section{Finite-Time Stationarity}

Proposition~\ref{prop:unbiased} identifies the joint update as stochastic gradient ascent on \eqref{eq:actual_problem}. Without additional reward-table structure, a natural unconditional guarantee is finite-time first-order stationarity. We first establish the smoothness of $J(\HH)$.

\begin{lemma}
\label{lem:smoothness}
The objective $J(\HH)$ is globally $L$-smooth. Specifically, if
\begin{equation}
d=\sum_{i=1}^{N}K_i,
\end{equation}
then a valid smoothness constant is $L=2d$.
Equivalently, for all preference vectors $\HH$ and $\HH'$, we have
\begin{equation}
    J(\HH')
\geq
J(\HH)+\langle \nabla J(\HH),\HH'-\HH\rangle
-\frac{L}{2}\|\HH'-\HH\|_2^2 .
\end{equation}
\end{lemma}

\begin{proof}
For coordinates $u=(i,x)$ and $v=(j,y)$, let $S_u(\ba)=\one\{A_i=x\}-\pi_i(x)$. Proposition~\ref{prop:unbiased} gives $\partial J/\partial H_i(x)=\E[r(\ba)S_u(\ba)]$. The Hessian entry indexed by $(v,u)$ is $\E[rS_uS_v]+\E[r\,\partial S_u/\partial H_j(y)]$, where $\partial S_u/\partial H_j(y)=-\one\{i=j\}\pi_i(x)(\one\{x=y\}-\pi_i(y))$. Both terms have absolute value at most $1$, so every Hessian entry is bounded by $2$. Since the Hessian is $d\times d$, $\|\nabla^2J(\HH)\|_2\le\|\nabla^2J(\HH)\|_F\le2d$, and the stated smoothness inequality follows.
\end{proof}

Combining Proposition~\ref{prop:unbiased} with Lemma~\ref{lem:smoothness} yields the following finite-time stationarity guarantee for OC-MPMABs.

\begin{theorem}[Finite-time stationarity]
\label{thm:stationarity}
Let
\begin{equation}
G^2=2N,
\,
\Gamma_T=\sum_{t=1}^{T}\gamma_t .
\end{equation}
Assume that $\gamma_t\geq 0$ for all $t$ and $\Gamma_T>0$. Then
\begin{equation}
    \frac{
\sum_{t=1}^{T}
\gamma_t
\E\!\left[\|\nabla J(\HH_t)\|_2^2\right]
}{
\Gamma_T
}
\leq
\frac{J^\star-J(\HH_1)}{\Gamma_T}
+
\frac{LG^2}{2}
\frac{\sum_{t=1}^{T}\gamma_t^2}{\Gamma_T}.
\end{equation}
Consequently,
\begin{equation}
    \min_{1\leq t\leq T}
\E\!\left[\|\nabla J(\HH_t)\|_2^2\right]
\leq
\frac{J^\star-J(\HH_1)}{\Gamma_T}
+
\frac{LG^2}{2}
\frac{\sum_{t=1}^{T}\gamma_t^2}{\Gamma_T}.
\end{equation}
In particular, if $\gamma_t=\eta>0$ for all $t$, then
\begin{equation}
    \frac{1}{T}
\sum_{t=1}^{T}
\E\!\left[\|\nabla J(\HH_t)\|_2^2\right]
\leq
\frac{J^\star-J(\HH_1)}{\eta T}
+
\frac{L\eta G^2}{2}.
\end{equation}
Further, let $\Delta_1=J^\star-J(\HH_1)$. If $\Delta_1>0$, choosing
$\eta=
\sqrt{
\frac{2\Delta_1}{LG^2T}
}$
yields
\begin{equation}
    \frac{1}{T}
\sum_{t=1}^{T}
\E\!\left[\|\nabla J(\HH_t)\|_2^2\right]
\leq
\sqrt{
\frac{2LG^2\Delta_1}{T}
}.
\end{equation}
If $\Delta_1=0$, the left-hand side is zero. Since $G^2=2N$ and $0\leq \Delta_1\leq 1$, the upper bound is at most
\begin{equation}
2\sqrt{\frac{LN}{T}}
=
2\sqrt{\frac{2dN}{T}} .
\end{equation}
\end{theorem}

\begin{proof}[Proof sketch]
Conditioning on the history $\F_t$, Proposition~\ref{prop:unbiased} implies that $\E[\mathbf g_t\mid\F_t]=\nabla J(\HH_t)$. The centered score vector for each player has squared norm at most $2$. Consequently, $\E[\|\mathbf g_t\|_2^2\mid\F_t]\le G^2=2N$. Applying Lemma~\ref{lem:smoothness} to the update $\HH_{t+1}=\HH_t+\gamma_t\mathbf g_t$ yields
\begin{equation}
    \E[J(\HH_{t+1})-J(\HH_t)\mid\F_t]\ge
\gamma_t\|\nabla J(\HH_t)\|_2^2-\frac{L\gamma_t^2G^2}{2}.
\end{equation}
Summing this inequality over $t$, using $J(\HH)\le J^\star$, and dividing by $\Gamma_T$ yields the stationarity bound. The constant-step-size rate follows by optimizing over $\eta$. The complete proof is provided in the supplementary material.
\end{proof}

Theorem~\ref{thm:stationarity} holds without additional structural assumptions: it requires only finite action spaces, the normalization $R_t\in[0,1]$, and the unbiased product-policy gradient identity. In particular, it does not require $J(\HH)$ to be concave, the optimal joint action to be unique, or the players to satisfy any coordination structure. The result provides finite-time control of the average and minimum expected gradient norm; it does not assert convergence of the iterates or global optimality. The next section explains why this distinction is unavoidable in general.

\section{Verifiable Structural Global Convergence}

Beyond unconditional stationarity, the results in this section separate two further layers: structural obstructions to global optimality and offline certificates that exclude those obstructions while controlling stochastic noise.

\subsection{Why Stationarity Does Not Imply Optimality}
\label{sec:trap-example}

First-order stationarity is insufficient for global optimality because, under the deterministic product-gradient flow $\dot{\HH}=\nabla J(\HH)$, a globally beneficial move may require several players to change arms jointly while every unilateral gradient points away from that move. The following two-player, three-arm example makes this coordination obstruction explicit:
\begin{equation}
\label{eq:reward_table}
\begin{array}{c|ccc}
 & A_2=0 & A_2=1 & A_2=2 \\ \hline
A_1=0 & 0.9 & 0.9 & 0\\
A_1=1 & 0.9 & 0.9 & 0\\
A_1=2 & 0 & 0 & 1 .
\end{array}
\end{equation}
Although the joint action $\ba^\star=(2,2)$ is globally optimal, the block $\{0,1\}^2$ is locally attracting under this deterministic product-gradient flow. Letting $q_i=\pi_i(2)$, the objective depends on $\HH$ only through $(q_1,q_2)$ and satisfies
\begin{equation}
J(\HH)=0.9(1-q_1)(1-q_2)+q_1q_2 .
\end{equation}
In contrast to the single-player bandit setting, the local advantage of arm $2$ is averaged with respect to the current policy of the other player and may therefore point away from the globally optimal profile. The following proposition formalizes this effect by identifying an invariant basin of attraction that contains the default initialization.

\begin{proposition}[Block-trap basin]
\label{prop:trap}
For the reward table in \eqref{eq:reward_table}, the region
\begin{equation}
q_1<\frac{9}{19},\quad q_2<\frac{9}{19}
\end{equation}
is forward-invariant under this flow, and every trajectory initialized within this region is attracted to the suboptimal block $\{0,1\}^2$. The default uniform initialization satisfies $q_1=q_2=1/3$ and thus lies within this basin of attraction.
\end{proposition}

\begin{proof}[Proof sketch]
Because the two suboptimal arms of each player constitute a block with reward $0.9$, the softmax gradient component that increases the probability assigned to the optimal arm of player~1 has the sign of $1.9q_2-0.9$, and the corresponding component for player~2 follows by symmetry. Consequently, $q_1=q_2=9/19$ is the stationary point that separates the two regimes in the reduced coordinate system. If both probabilities are below $9/19$, both gradient components are negative. Therefore, the flow decreases both optimal-arm probabilities and cannot escape the lower-left region. The complete derivation is provided in the supplementary material.
\end{proof}

Proposition~\ref{prop:trap} shows that Theorem~\ref{thm:stationarity} does not imply global optimality. The obstruction is structural: the globally optimal joint action consists of one local arm per player, and any such arm may have a negative product-policy advantage unless the other players already coordinate with it. The following proposition shows that this limitation persists beyond the deterministic basin analysis.

\begin{proposition}[Trap-minor obstruction]
\label{prop:trap-structure}
Suppose that a deterministic two-player reward table contains a $2\times2$ minor with a corner $(x_0,y_0)$ whose reward is below $J^\star$ but strictly exceeds both adjacent entries. If this corner also strictly dominates every unilateral deviation outside the minor, then its pure product policy has an open basin of attraction $\mathcal B$.
For every compact set $K\subset\mathcal B$ and every fixed $\alpha>0$, there exist $C_K<\infty$ and $t_K(\alpha)<\infty$ such that, for any finite-preference initialization that induces a product policy in $K$, the update in~\eqref{eq:update} with $\gamma_t=\alpha/(t+t_0)$ converges to the trap with probability at least
\begin{equation}
\left[1-\frac{C_K\alpha^2}{t_0}\right]_+
\end{equation}
for every $t_0\ge t_K(\alpha)$, where $[z]_+=\max\{z,0\}$.
\end{proposition}
The supplementary material provides complete proofs and the corresponding asymptotic lower bound on the expected gap. Together, these results constitute the negative component of the framework: the obstruction is a reusable structural pattern rather than a simulation artifact. They also identify precisely why the single-player softmax-gradient mechanism fails in the multi-player setting. Given the current policies of the other players, a local component of a globally optimal joint action may have a lower expected reward than a suboptimal local arm until the other players coordinate with that component. Consequently, finite action spaces, normalized rewards, and the oblivious gradient-bandit update are insufficient, on their own, to guarantee global optimality.

\paragraph{Default-initialization block trap.}
The tied block in Proposition~\ref{prop:trap} is not directly covered by the strict trap-minor result in Proposition \ref{prop:trap-structure}. Nevertheless, for deterministic rewards, the supplementary material proves that there exist constants $C_0<\infty$ and $\alpha_0>0$, with $C_0\alpha_0^2\le1/2$, such that, for every $\alpha\in(0,\alpha_0]$, the update initialized with zero preferences and run with $\gamma_t=\alpha/(t+1)$ satisfies
\begin{equation}
\Pr\{\pi_{i,t}(2)\to0,\ i=1,2\}
\ge[1-C_0\alpha^2]_+
\ge\frac12.
\end{equation}
Thus, the stochastic update converges to the suboptimal block with a probability that is lower-bounded independently of the time horizon.

\subsection{Two Structural Certificates}

To recover guarantees of global optimality without modifying the gradient update rule, we restrict our analysis to reward tables that satisfy two properties within a region containing the trajectory. First, every suboptimal product policy must have a gradient whose magnitude is sufficient to control its optimality gap, thereby excluding the traps identified above. Second, the update variance must decrease with the gradient magnitude so that sampling noise does not dominate the resulting ascent. We formalize these requirements through product-gradient-dominance and self-bounding-noise certificates. Together, these certificates yield bounds on the product-policy gap and pseudo-regret under the specified conditions.

Fix an optimal joint action $\ba^\star=(A_1^\star,\ldots,A_N^\star)$. For a product policy $\pi=(\pi_1,\ldots,\pi_N)$, define
\begin{align}
J(\pi)
&=\sum_{\ba\in\A}
\left(\prod_{i=1}^{N}\pi_i(A_i)\right)r(\ba),\\
Q_i^\pi(x) 
&=\sum_{\ba_{-i}\in\A_{-i}}
\left(\prod_{j\neq i}\pi_j(A_j)\right)r(x,\ba_{-i}).
\end{align}
Under the product-softmax parameterization, the preference gradient has coordinates
\begin{equation}
\mathcal G_{i,x}(\pi)
=\pi_i(x)\bigl(Q_i^\pi(x)-J(\pi)\bigr),
\end{equation}
and Proposition~\ref{prop:unbiased} gives $\mathcal G(\pi_{\HH})=\nabla J(\HH)$. For the baseline-free update, define the second-moment envelope
\begin{equation}
\mathcal V(\pi)
=\sum_{\ba\in\A}
\left(\prod_{i=1}^{N}\pi_i(A_i)\right)r(\ba)
\sum_{i=1}^{N}\|\mathbf e_{A_i}-\pi_i\|_2^2,
\end{equation}
where $\mathbf e_{A_i}\in\R^{K_i}$ denotes the unit vector corresponding to $A_i$. Since $R_t^2\le R_t$, the conditional second moment of the update satisfies
$\E[\|\mathbf g_t\|_2^2\mid\F_t]\le\mathcal V(\pi_{\HH_t})$.
For $c\in(0,1]$, define
\begin{equation}
\begin{aligned}
\Pi_\star(c)
&=\{\pi:\pi_i(A_i^\star)\ge c,\ \forall i\},\\
\Pi_{\rm all}(c)
&=\{\pi:\pi_i(x)\ge c,\ \forall i,x\}.
\end{aligned}
\end{equation}
The all-arm region is considered only when it is nonempty.

\paragraph{Product-gradient-dominance certificate.}
The deterministic certificate is
\begin{equation}
\lambda_\star(c)
=\inf_{\substack{\pi\in\Pi_\star(c)\\J(\pi)<J^\star}}
\frac{\|\mathcal G(\pi)\|_2^2}{(J^\star-J(\pi))^2}.
\end{equation}
Positive $\lambda_\star(c)$ rules out suboptimal stationary points in $\Pi_\star(c)$ by requiring the gradient magnitude to control the optimality gap.

\paragraph{Self-bounding-noise certificate.}
The stochastic certificate is
\begin{equation}
M_{\rm all}(c)
=\sup_{\substack{\pi\in\Pi_{\rm all}(c)\\\|\mathcal G(\pi)\|_2>0}}
\frac{\mathcal V(\pi)}{\|\mathcal G(\pi)\|_2^2}.
\end{equation}
Set $M_{\rm all}(c)=\infty$ if $\mathcal V>0$ at any zero-gradient policy, set it to $0$ if the displayed index set is empty, and otherwise use the displayed supremum. Thus, finite $M_{\rm all}(c)$ gives $\mathcal V(\pi)\le M_{\rm all}(c)\|\mathcal G(\pi)\|_2^2$ throughout $\Pi_{\rm all}(c)$. Both certificates are offline reward-table properties rather than assumptions fitted to a trajectory.
Requiring these controls over the entire policy space may be unnecessarily restrictive. If a deterministic region $\mathcal S$ is certified to contain the iterates, it suffices to use the localized constants
\begin{equation}
\begin{aligned}
\lambda_{\mathcal S}
&=\inf_{\substack{\pi\in\mathcal S\\J(\pi)<J^\star}}
\frac{\|\mathcal G(\pi)\|_2^2}{(J^\star-J(\pi))^2},\\
M_{\mathcal S}
&=\sup_{\substack{\pi\in\mathcal S\\\|\mathcal G(\pi)\|_2>0}}
\frac{\mathcal V(\pi)}{\|\mathcal G(\pi)\|_2^2}.
\end{aligned}
\end{equation}
Use the same zero-gradient and empty-supremum conventions for $M_{\mathcal S}$. These constants suffice whenever $\mathcal S$ contains the trajectory. The next subsection applies the all-arm certificates to a finite-horizon region.

\subsection{Structural Global Rate and Regret}

When both certificates hold on a region containing the trajectory, product-gradient dominance converts every positive optimality gap into a quantitative ascent signal, while self-bounding noise ensures that the stochastic second-moment penalty remains proportional to this signal. With a sufficiently small stepsize, smoothness then guarantees a decreasing expected gap and establishes pseudo-regret bounds; uniform certificates further imply that the expected value of the product policy converges to $J^\star$. We first construct a certified finite-horizon region that contains every iterate.

\begin{lemma}[Finite-horizon exploration certificate]
\label{lem:exploration}
Fix a horizon $T$, and let
\begin{equation}
\Gamma_T=\sum_{t=1}^{T}\gamma_t,
\,
K_{\max}=\max_{i\in[N]}K_i,
\,
c_T=\frac{e^{-2\Gamma_T}}{K_{\max}}.
\end{equation}
If $H_{i,1}(x)=0$ for all $i$ and $x$, then, for every $t\le T+1$, every player $i$, and every local arm $x\in\A_i$,
\begin{equation}
\pi_{i,t}(x)\ge c_T.
\end{equation}
\end{lemma}

\begin{proof}
Since $R_t\in[0,1]$ and the centered score lies in $[-1,1]$, $|g_{i,t}(x)|\le1$. Hence, for $t\le T+1$, $|H_{i,t}(x)|\le\sum_{s=1}^{t-1}\gamma_s\le\Gamma_T$. The softmax formula then gives $\pi_{i,t}(x)\ge e^{-\Gamma_T}/(K_i e^{\Gamma_T})\ge e^{-2\Gamma_T}/K_{\max}=c_T$.
\end{proof}

Lemma~\ref{lem:exploration} therefore places the entire finite-horizon trajectory in
$\Pi_{\rm all}(c_T)\subseteq\Pi_\star(c_T)$, where both certificates can be evaluated. To state the resulting performance guarantee, define the product-policy gap by
\begin{equation}
D_t=J^\star-J(\pi_{\HH_t}),
\quad
d_t=\E[D_t],
\end{equation}
and define the cumulative mean-reward regret random variable, whose expectation is the expected pseudo-regret, as
\begin{equation}
\operatorname{Reg}(T)=
\sum_{t=1}^{T}
\bigl(J^\star-r(\ba_t)\bigr).
\end{equation}
On this certified region, $\lambda_\star$ converts gradient progress into gap reduction, while $M_{\rm all}$ controls stochastic update noise. These two controls yield the following finite-horizon guarantee.

\begin{theorem}[Certificate-conditioned finite-horizon rate and pseudo-regret]
\label{thm:structural}
Fix a time horizon $T$. Starting from $H_{i,1}(x)=0$ for all $i$ and $x$, run the update in~\eqref{eq:update} using the step sizes $\{\gamma_t\}_{t=1}^{T}$, and define
\begin{equation}
c_T=\frac{e^{-2\Gamma_T}}{K_{\max}},
\quad
\lambda_T=\lambda_\star(c_T),
\quad
M_T=M_{\rm all}(c_T).
\end{equation}
Assume that $\lambda_T>0,
0<M_T<\infty$, 
and $0\le \gamma_t\le \frac{1}{LM_T},
\, t=1,\ldots,T$.
We further define
\begin{equation}
\rho_t=
\gamma_t\left(1-\frac{LM_T\gamma_t}{2}\right).
\end{equation}
If $D_1>0$, then, for every $t\le T+1$,
\begin{equation}
d_t
\le
\frac{1}
{D_1^{-1}+\lambda_T\sum_{s=1}^{t-1}\rho_s}.
\end{equation}
Moreover,
\begin{equation}
\E[\operatorname{Reg}(T)]
\le
\sum_{t=1}^{T}
\frac{1}
{D_1^{-1}+\lambda_T\sum_{s=1}^{t-1}\rho_s}.
\end{equation}
\end{theorem}

% \begin{proof}[Proof sketch]
% If $D_1=0$, the product-policy gap is already zero, and the regret bound is zero. Now consider $D_1>0$.
% Lemma~\ref{lem:exploration} keeps the trajectory inside $\Pi_{\rm all}(c_T)$, and therefore also inside $\Pi_\star(c_T)$. The two certificates imply
% \begin{equation}
% \begin{aligned}
% \|\nabla J(\HH_t)\|_2^2
% &\ge \lambda_TD_t^2,\\
% \E[\|\mathbf g_t\|_2^2\mid\F_t]
% &\le M_T\|\nabla J(\HH_t)\|_2^2 .
% \end{aligned}
% \end{equation}
% Combining these inequalities with the smoothness-based ascent step gives
% $\E[D_{t+1}\mid\F_t]\le D_t-\rho_t\lambda_TD_t^2$. The Jensen inequality gives $d_{t+1}\le d_t-\rho_t\lambda_Td_t^2$. If $d_{t+1}=0$, the bound follows immediately. Otherwise, $1/d_{t+1}-1/d_t\ge(d_t-d_{t+1})/d_t^2\ge\rho_t\lambda_T$, and summing yields the reciprocal bound. The conditional expected regret equals the current product-policy gap, so summing the same bound proves the regret statement.
% \end{proof}

The complete proof is given in the supplementary material. Theorem~\ref{thm:structural} is an offline finite-horizon guarantee: choose $T$ and a stepsize schedule, certify $(\lambda_T,M_T)$ on the induced region, and verify the stepsize condition. For a constant stepsize $\gamma_t=\eta$ with $0<\eta\le1/(LM_T)$, let
\begin{equation}
\beta_T
=
\lambda_T\eta\left(1-\frac{LM_T\eta}{2}\right).
\end{equation}
Then
\begin{equation}
d_t\le
\frac{1}{D_1^{-1}+\beta_T(t-1)}
\end{equation}
and
\begin{equation}
\E[\operatorname{Reg}(T)]
\le
D_1+
\frac{1}{\beta_T}
\log\bigl(1+\beta_TD_1(T-1)\bigr).
\end{equation}
Thus, uniform certificates yield $d_t=O(1/t)$ and $\E[\operatorname{Reg}(T)]=O(\log T)$; otherwise, the guarantee remains horizon dependent. The next subsection identifies sparse and dense reward classes with analytic certificates.

\begin{figure*}[t]
\centering
\includegraphics[width=\textwidth]{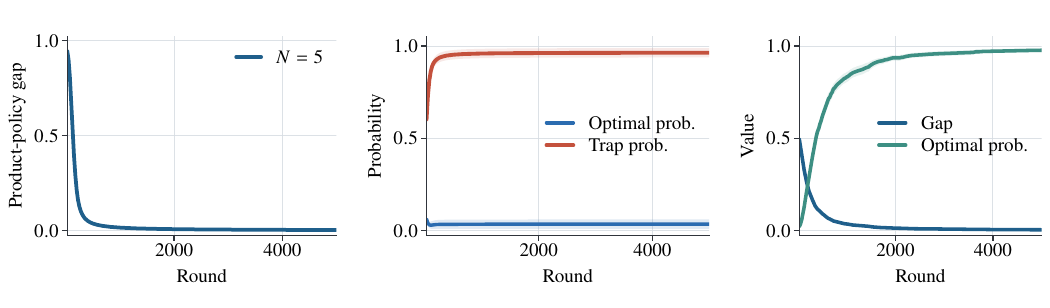}
\caption{Sparse coordination (left), block trap (center), and AI routing (right); bands show 95\% CIs.}
\label{fig:main-experiments}
\end{figure*}

\subsection{Explicit Verifiable Classes}

To make Theorem~\ref{thm:structural} operational, we derive analytic certificates for two nontrivial classes of reward tables: sparse coordination tables admit a trajectory-invariant region that yields uniform rates, whereas dense locally dominant tables admit explicit finite-horizon certificates.

\begin{corollary}[Optimally supported common rewards]
\label{cor:optimally-supported}
Suppose there is a unique optimal joint action
\begin{equation}
\ba^\star=(A_1^\star,\ldots,A_N^\star)
\end{equation}
such that
\begin{equation}
r(\ba^\star)=J^\star>0,
\qquad
r(\ba)=0
\quad
\forall \ba\neq\ba^\star.
\end{equation}
Let $c_\star=1/K_{\max}$. For the trajectory of \eqref{eq:update} initialized with $H_{i,1}(x)=0$,
\begin{equation}
\pi_{i,t}(A_i^\star)\ge \frac{1}{K_i}
\quad
\text{for all } i \text{ and } t.
\end{equation}
Let $\mathcal S_{\rm sp}=\{\pi:\pi_i(A_i^\star)\ge1/K_i,\ \forall i\}$ and $M_{\rm sp}=M_{\mathcal S_{\rm sp}}$. Every iterate lies in $\mathcal S_{\rm sp}$ and assigns probability at least $c_\star^N$ to $\ba^\star$, yielding
\begin{equation}
\lambda_\star(c_\star)
\ge
\frac{c_\star^{2N}}{N^2},
\qquad
M_{\rm sp}
\le
\frac{1}{J^\star c_\star^N}.
\end{equation}
Consequently, for any constant stepsize
\begin{equation}
0<\eta\le \frac{J^\star c_\star^N}{L},
\end{equation}
update \eqref{eq:update} satisfies $d_t=O(1/t)$ and
$\E[\operatorname{Reg}(T)]=O(\log T)$.
\end{corollary}

% \begin{proof}[Proof sketch]
% Only $\ba^\star$ yields a positive reward. Thus, non-optimal samples leave the preferences unchanged, whereas samples of $\ba^\star$ weakly increase the preference of every optimal local arm, implying $\pi_{i,t}(A_i^\star)\ge1/K_i$. On this invariant region, $J(\pi)=J^\star\prod_i\pi_i(A_i^\star)$, and a direct calculation yields the stated gradient and noise-ratio bounds. The details are provided in the supplementary material.
% \end{proof}

The complete proof is given in the supplementary material. This corollary relies on the sparsity of the reward table. The following proposition removes this requirement for dense reward tables in which the optimal local arms dominate all alternatives under every opponent profile, thereby excluding the coordination reversal identified in Proposition~\ref{prop:trap}.

\begin{proposition}[Local dominance certificates]
\label{prop:local-dominance}
Suppose the optimal joint action
$\ba^\star=(A_1^\star,\ldots,A_N^\star)$
satisfies uniform local dominance: for some $\delta>0$, every player $i$, arm $x\neq A_i^\star$, and opponent profile $\ba_{-i}$ obey
$r(A_i^\star,\ba_{-i})-r(x,\ba_{-i})\ge \delta$. Then, for every $c\in(0,1]$, $\lambda_\star(c)\ge(c\delta/N)^2$. If $c\le1/K_{\max}$ and at least one player has a non-optimal arm, define $\mu(c)=1-\prod_i(1-(K_i-1)c)$. Then
\begin{equation}
M_{\rm all}(c)\le \frac{2N^3}{c^2\delta^4\mu(c)^2}.
\end{equation}
Thus every locally dominant dense table has explicit finite-horizon certificates for Theorem~\ref{thm:structural}.
\end{proposition}

% \begin{proof}[Proof sketch]
% Local dominance turns any positive product-policy gap into an optimal-arm gradient coordinate of magnitude at least $(c\delta/N)(J^\star-J(\pi))$. On $\Pi_{\rm all}(c)$, the gap is at least $\delta\mu(c)$, while $\mathcal V(\pi)\le2N$. The full proof is provided in the supplementary material.
% \end{proof}

A complete proof of this proposition is provided in the supplementary material. Note that Proposition~\ref{prop:local-dominance} covers dense, nonadditive reward tables. For example, if the normalized reward is $r(\ba)=N^{-1}\sum_i u_i(A_i)+\epsilon h(\ba)$, each $u_i$ has an optimal-arm gap of at least $\Delta$, and all one-coordinate changes in $h$ are bounded in absolute value by $B$, then $\Delta/N>|\epsilon|B$ yields a local-dominance margin $\delta=\Delta/N-|\epsilon|B$.
Thus, Corollary~\ref{cor:optimally-supported} and Proposition~\ref{prop:local-dominance} avoid the block trap through complementary mechanisms: joint reinforcement in sparse tables and persistent local advantage in dense tables. Figure~\ref{fig:main-experiments} illustrates both mechanisms, together with the trap that arises under oblivious feedback.

\section{Mechanism Diagnostics}

The diagnostics illustrate the theoretical mechanisms under the fully oblivious protocol rather than provide a comprehensive benchmark. The left panel of Figure~\ref{fig:main-experiments} shows a decreasing product-policy gap for the sparse coordination class and its invariant support geometry. Because the practical stepsize exceeds the conservative range in Corollary~\ref{cor:optimally-supported}, this trajectory diagnoses the mechanism rather than verifies the certified constants. The center panel shows the negative case: from uniform initialization, the optimal-profile probability remains small while the suboptimal block dominates, consistent with Proposition~\ref{prop:trap}. The right panel presents a dense, nonadditive AI-pipeline case with an analytic local-dominance certificate; its four modules observe only one terminal score. Supplementary diagnostics show that certificate estimates are associated with terminal gaps and that the trap basin expands as the suboptimal block reward approaches the optimum.

\section{Related Work}

\textbf{Multi-player and cooperative bandits.} Multi-player bandits study collisions, decentralized allocation, and limited communication \cite{anantharam1987asymptotically,liu2010distributed,anandkumar2011distributed,kalathil2014decentralized,rosenski2016multi,besson2018multi,boursier2019sic,bistritz2018distributed}. Cooperative and combinatorial bandits use communication, aggregation, centralization, or component feedback \cite{bistritz2020cooperative,gai2012combinatorial,chen2013combinatorial,kveton2015tight}. Unlike these settings, OC-MPMAB provides neither coordination signals nor awareness of other learners; we characterize the hidden product-policy geometry induced by independent updates driven solely by a shared scalar reward.

\textbf{Common-interest games.} Product-policy traps are related to common-payoff and potential games \cite{monderer1996potential,marden2009payoff}, learning automata and payoff-based dynamics \cite{narendra1989learning,leslie2005individual,mertikopoulos2016learning}, and independent multi-agent reinforcement learning coordination failures \cite{claus1998dynamics,lauer2000algorithm,littman1994markov,hu2003nash,shoham2007multiagent,busoniu2008comprehensive}. Unlike equilibrium-selection analyses with a known game form or prescribed dynamics, we characterize stationarity, traps, and finite-table certificates when learners observe neither joint actions, player identities, nor the game model.

\textbf{Policy gradients and gradient bandits.} Likelihood-ratio and policy-gradient identities are classical \cite{glynn1990likelihood,williams1992simple,sutton1999policy,baxter2001infinite,konda2003actor}; baselines and variance reduction are well studied \cite{greensmith2004variance,mei2022role}. Softmax policy-gradient theory mostly treats single-agent or centralized policies \cite{mei2020global,mei2020escaping,li2021softmax,agarwal2021theory,mei2023stochastic}. Unlike the single-player case, hidden product coupling can make the component $A_i^\star$ disadvantageous because $Q_i^\pi(A_i^\star)$ averages over the policies of other players. We characterize the resulting traps and verifiable conditions for global rates, rather than deriving another score-function identity.

\section{Conclusion}

This paper presents a theoretical analysis of OC-MPMAB, a minimal multi-player bandit model in which fully oblivious players observe only their selected arms and a shared reward. Our findings are threefold. First, independent gradient-bandit updates jointly form an unbiased stochastic gradient estimator of the hidden product-policy objective, which guarantees universal finite-time convergence to first-order stationary points. Second, first-order stationarity cannot generally be strengthened to global optimality: product-policy traps can induce suboptimal attractors, and a deterministic block trap can lead to failure with constant probability from the default initialization under sufficiently small diminishing stepsizes. Third, structural properties of the reward table can recover global optimality guarantees. Specifically, finite-table certificates for product-gradient dominance and self-bounding stochastic noise provide certificate-dependent bounds on the product-policy gap and pseudo-regret; uniform certificates further achieve an $O(1/t)$ gap and $O(\log T)$ pseudo-regret.

\bibliography{aaai2027}


\end{document}

\maketitle

\appendix
\section{Omitted Proofs}
\label{app:proofs}

Unless redefined locally, all model, policy, and certificate notation follows the main text.

\subsection{Proof of Main-Text Proposition 1}

\begin{proof}
Fix an arbitrary player $i$ and arm $x\in\A_i$, and condition on $\F_t$. Because only the local factor $\pi_i(A_i)$ depends on $H_i(x)$,
\begin{align}
  \frac{\partial J(\HH)}{\partial H_i(x)}
  &=
  \sum_{\ba\in\A}r(\ba)
 \left(\prod_{j\ne i}\pi_j(A_j)\right)
  \frac{\partial\pi_i(A_i)}{\partial H_i(x)} .
\end{align}
Let $Z_i$ denote the softmax denominator. We then have
\begin{equation}
\pi_i(A_i)=\frac{\exp(H_i(A_i))}{Z_i},
\end{equation}
and the quotient rule yields
\begin{equation}
  \frac{\partial \pi_i(A_i)}{\partial H_i(x)}
  =
\pi_i(A_i)\bigl(\one\{A_i=x\}-\pi_i(x)\bigr).
\end{equation}
Thus
\begin{equation}
  \frac{\partial J(\HH)}{\partial H_i(x)}
  =
  \sum_{\ba\in\A}\pi_{\HH}(\ba)r(\ba)
  \bigl(\one\{A_i=x\}-\pi_i(x)\bigr).
\end{equation}
Meanwhile,
\begin{align}
\E[g_{i,t}(x)\mid\F_t]
&=
\E\!\left[
R_t\bigl(\one\{A_{i,t}=x\}-\pi_{i,t}(x)\bigr)
\right]\notag\\
&=
\sum_{\ba\in\A}\pi_{\HH_t}(\ba)r(\ba)
  \bigl(\one\{A_i=x\}-\pi_{i,t}(x)\bigr),
\end{align}
which is precisely the derivative evaluated at $\HH=\HH_t$.
\end{proof}

\subsection{Proof of Main-Text Theorem 1}

\begin{proof}
Let $\F_t$ denote the history before round $t$, and write
\begin{equation}
\mathbf g_t
=
\bigl(g_{i,t}(x): i\in[N], x\in\A_i\bigr)\in\R^d .
\end{equation}
Here $d=\sum_i K_i$, and we set $\Delta_1=J^\star-J(\HH_1)$.
Conditional on $\F_t$, Main-Text Proposition 1 gives
\begin{equation}
\E[\mathbf g_t\mid\F_t]
=
\nabla J(\HH_t).
\end{equation}
For each player, the centered score vector
\begin{equation}
s_{i,t}(x)=\one\{A_{i,t}=x\}-\pi_{i,t}(x)
\end{equation}
satisfies $\|s_{i,t}\|_2^2\le2$. Since $R_t\in[0,1]$,
\begin{equation}
\|\mathbf g_t\|_2^2
=
\sum_{i=1}^{N}R_t^2\|s_{i,t}\|_2^2
\le
2N
=
G^2.
\end{equation}
Applying Main-Text Lemma 1 to
$\HH_{t+1}=\HH_t+\gamma_t\mathbf g_t$ yields
\begin{equation}
J(\HH_{t+1})
\ge
J(\HH_t)+
\gamma_t\langle\nabla J(\HH_t),\mathbf g_t\rangle
-
\frac{L}{2}\gamma_t^2\|\mathbf g_t\|_2^2 .
\end{equation}
Taking conditional expectation yields
\begin{equation}
\E[J(\HH_{t+1})\mid\F_t]
\ge
J(\HH_t)+
\gamma_t\|\nabla J(\HH_t)\|_2^2
-
\frac{L}{2}\gamma_t^2G^2 .
\end{equation}
Rearranging, taking total expectations, and summing over $t=1,\ldots,T$ yields
\begin{align}
\sum_{t=1}^{T}
\gamma_t
\E\!\left[\|\nabla J(\HH_t)\|_2^2\right]
&\le
\E[J(\HH_{T+1})]-J(\HH_1)
\notag\\
&\quad+
\frac{LG^2}{2}\sum_{t=1}^{T}\gamma_t^2 .
\end{align}
Since $J(\HH)\le J^\star$, dividing by $\Gamma_T$ proves the weighted-average bound. The minimum bound follows because the minimum is no greater than the weighted average. For the constant stepsize $\gamma_t=\eta$, substitution gives
\begin{equation}
\frac{1}{T}\sum_{t=1}^{T}
\E\!\left[\|\nabla J(\HH_t)\|_2^2\right]
\le
\frac{\Delta_1}{\eta T}
+
\frac{L\eta G^2}{2}.
\end{equation}
When $\Delta_1>0$, the stated choice of $\eta$ balances the two terms. If $\Delta_1=0$, the full support of the softmax policy implies that every joint action has mean reward $J^\star$, so $\nabla J$ is identically zero. Finally, $G^2=2N$, $L=2d$, and $0\le\Delta_1\le1$.
\end{proof}

\subsection{Proof of Main-Text Proposition 2}

\begin{proof}
For the two-player, three-arm block-trap table in Main-Text Proposition 2, let
\begin{equation}
q_i=\pi_i(2),\qquad i=1,2 .
\end{equation}
The objective depends on the preferences only through $(q_1,q_2)$. Define its reduced form by
\begin{equation}
J_{\rm trap}(q_1,q_2):=J(\HH)
=0.9(1-q_1)(1-q_2)+q_1q_2 .
\end{equation}
The two non-optimal arms of each player form a reward-$0.9$ block. Thus, only the total probability $q_i$ assigned to the optimal arm matters in this two-dimensional reduction. For player 1, the conditional value of arm $2$ is $q_2$, whereas each arm in the non-optimal block has conditional value $0.9(1-q_2)$. Therefore, the product-gradient direction that increases the probability of the optimal arm of player 1 relative to the non-optimal block has the sign of
\begin{equation}
q_2-0.9(1-q_2)=1.9q_2-0.9 .
\end{equation}
Similarly, the corresponding direction for player 2 has the sign of $1.9q_1-0.9$. Hence, the separating interior stationary point in the reduced $(q_1,q_2)$ coordinates is
\begin{equation}
q_1=q_2=\frac{0.9}{1.9}=\frac{9}{19}.
\end{equation}
In the full product-softmax parameterization, this point corresponds to a stationary set because the reward table is invariant to how each player distributes the remaining probability $1-q_i$ between arms $0$ and $1$.
At any policy in this stationary set,
\begin{equation}
J
=
0.9\left(\frac{10}{19}\right)^2
+
\left(\frac{9}{19}\right)^2
=
\frac{171}{361}
<1
=J^\star .
\end{equation}
Thus, product-gradient dominance fails for every $c\in(0,9/19]$.

For completeness, the exact policy-space flow makes the boundary behavior explicit. Under $\dot{\HH}=\nabla J(\HH)$, the softmax chain rule yields, for $j\neq i$,
\begin{equation}
\dot q_i
=
q_i^2
\left((1-q_i)^2+\pi_i(0)^2+\pi_i(1)^2\right)
\left(1.9q_j-0.9\right).
\end{equation}
Thus, if $q_1<9/19$ and $q_2<9/19$, both optimal-arm probabilities decrease, and the lower-left region is forward invariant. Each $q_i$ therefore has a limit. If either limit were positive, the displayed derivative would remain bounded above by a negative constant because the other probability never increases, yielding a contradiction. Hence, $q_1,q_2\to0$. The default initialization satisfies $q_1=q_2=1/3<9/19$ and therefore lies in this attracting basin.
\end{proof}

\subsection{Proof of Main-Text Proposition 3}

\begin{proof}
Supplementary Proposition~\ref{prop:trap-minor} establishes an open attracting basin for a strict trap minor, while Supplementary Proposition~\ref{prop:trap-minor-stochastic} establishes stochastic lock-in on compact subsets of that basin. Together, these results prove Main-Text Proposition 3.
\end{proof}

\subsection{Proof of the Main-Text Default-Initialization Block-Trap Statement}

\begin{proof}
The probability bound under default initialization and the corresponding asymptotic lower bound on the expected gap follow from Supplementary Proposition~\ref{prop:stochastic-trap}.
\end{proof}

\subsection{Proof of Main-Text Theorem 2}

\begin{proof}
By Main-Text Lemma 2, $\pi_{\HH_t}\in\Pi_{\rm all}(c_T)$ for all $t\le T+1$, and hence $\pi_{\HH_t}\in\Pi_\star(c_T)$. When $D_t>0$, the definition of $\lambda_T=\lambda_\star(c_T)$ gives
\begin{equation}
\|\nabla J(\HH_t)\|_2^2
=
\|\mathcal G(\pi_{\HH_t})\|_2^2
\ge
\lambda_TD_t^2 .
\end{equation}
The second-moment envelope, together with $M_T=M_{\rm all}(c_T)$, gives
\begin{equation}
\E[\|\mathbf g_t\|_2^2\mid\F_t]
\le
M_T\|\nabla J(\HH_t)\|_2^2 .
\end{equation}
Applying Main-Text Lemma 1 and using the conditional unbiasedness from Main-Text Proposition 1 and the definition of $\rho_t$, we obtain
\begin{align}
\E[J(\HH_{t+1})-J(\HH_t)\mid\F_t]
&\ge
\gamma_t\|\nabla J(\HH_t)\|_2^2
\notag\\
&\quad
-\frac{L}{2}\gamma_t^2
\E[\|\mathbf g_t\|_2^2\mid\F_t]\notag\\
&\ge
\rho_t\|\nabla J(\HH_t)\|_2^2
\ge
\rho_t\lambda_TD_t^2 .
\end{align}
Thus
\begin{equation}
\E[D_{t+1}\mid\F_t]
\le
D_t-\rho_t\lambda_TD_t^2 .
\end{equation}
Taking expectations and applying Jensen's inequality gives
\begin{equation}
d_{t+1}\le d_t-\rho_t\lambda_Td_t^2 .
\end{equation}
If $d_{t+1}=0$, the reciprocal bound follows immediately. Otherwise, the sequence $\{d_t\}$ is nonincreasing, so
\begin{equation}
\frac1{d_{t+1}}-\frac1{d_t}
\ge
\frac{d_t-d_{t+1}}{d_t^2}
\ge
\rho_t\lambda_T .
\end{equation}
Because the zero-preference initialization is deterministic, $d_1=D_1$. Summing from $1$ to $t-1$ proves the gap bound. For the pseudo-regret,
\begin{equation}
\E[J^\star-r(\ba_t)\mid\F_t]
=
J^\star-J(\pi_{\HH_t})
=D_t.
\end{equation}
Taking expectations and summing yields
\begin{equation}
\E[\operatorname{Reg}(T)]=\sum_{t=1}^T d_t,
\end{equation}
and substituting the gap bound proves the displayed pseudo-regret bound.
\end{proof}

\subsection{Proof of Main-Text Corollary 1}

\begin{proof}
Because rewards lie in $[0,1]$, the condition $r(\ba)=0$ implies that $R_t=0$ almost surely whenever $\ba_t\neq\ba^\star$. No preference changes in such rounds. If $\ba_t=\ba^\star$, then, for any $x\neq A_i^\star$,
\begin{equation}
g_{i,t}(A_i^\star)-g_{i,t}(x)
=
R_t\bigl(1-\pi_{i,t}(A_i^\star)+\pi_{i,t}(x)\bigr)
\ge0 .
\end{equation}
Thus, $H_{i,t}(A_i^\star)-H_{i,t}(x)$ never decreases from its initial value $0$, which implies $\pi_{i,t}(A_i^\star)\ge1/K_i$.

Fix any $\pi\in\Pi_\star(c_\star)$, and let
\begin{equation}
p_i=\pi_i(A_i^\star),
\qquad
p_\star=\prod_{i=1}^{N}p_i .
\end{equation}
Then $J(\pi)=J^\star p_\star$ and, writing $D(\pi)=J^\star-J(\pi)$, $D(\pi)=J^\star(1-p_\star)$. The optimal-arm gradient coordinate is
\begin{equation}
\mathcal G_{i,A_i^\star}(\pi)
=
J^\star p_\star(1-p_i).
\end{equation}
Since $1-p_\star\le\sum_i(1-p_i)$, there exists a player such that
\begin{equation}
1-p_i\ge \frac{D(\pi)}{J^\star N}.
\end{equation}
Using $p_\star\ge c_\star^N$,
\begin{equation}
\|\mathcal G(\pi)\|_2
\ge
J^\star p_\star(1-p_i)
\ge
\frac{c_\star^N}{N}D(\pi),
\end{equation}
which proves the lower bound on $\lambda_\star(c_\star)$.

For the noise certificate on $\mathcal S_{\rm sp}$,
\begin{equation}
\mathcal G_i(\pi)
=
J^\star p_\star(\mathbf e_{A_i^\star}-\pi_i),
\end{equation}
so
\begin{equation}
\|\mathcal G(\pi)\|_2^2
=
(J^\star)^2p_\star^2
\sum_i\|\mathbf e_{A_i^\star}-\pi_i\|_2^2 .
\end{equation}
The same sparse reward structure gives
\begin{equation}
\mathcal V(\pi)
=
J^\star p_\star
\sum_i\|\mathbf e_{A_i^\star}-\pi_i\|_2^2 .
\end{equation}
Whenever $\|\mathcal G(\pi)\|_2>0$,
\begin{equation}
\frac{\mathcal V(\pi)}
{\|\mathcal G(\pi)\|_2^2}
=
\frac{1}{J^\star p_\star}
\le
\frac{1}{J^\star c_\star^N}.
\end{equation}
If $\mathcal G(\pi)=0$ on $\mathcal S_{\rm sp}$, the displayed gradient formula forces every $\pi_i$ to be concentrated on $A_i^\star$, and hence $\mathcal V(\pi)=0$. Thus the zero-gradient convention for $M_{\rm sp}$ is also satisfied.
The rate statement follows by applying the same one-step recursion as in Main-Text Theorem 2 within this invariant region, using constants that are independent of the horizon.
\end{proof}

\subsection{Proof of Main-Text Proposition 4}

\begin{proof}
Let $D(\pi)=J^\star-J(\pi)$ and $q_i=1-\pi_i(A_i^\star)$. Because the rewards are normalized and $\ba^\star$ is optimal,
\begin{equation}
D(\pi)=J^\star-J(\pi)
\le
\Pr_{\ba\sim\pi}\{\ba\neq\ba^\star\}
\le
\sum_{i=1}^{N}q_i .
\end{equation}
Thus, some player $i$ satisfies $q_i\ge D(\pi)/N$. For this player,
\begin{equation}
J(\pi)=\sum_{x\in\A_i}\pi_i(x)Q_i^\pi(x),
\end{equation}
and the local-dominance condition implies
\begin{equation}
Q_i^\pi(A_i^\star)-J(\pi)
\ge
\delta q_i .
\end{equation}
If $\pi\in\Pi_\star(c)$, then
\begin{equation}
\mathcal G_{i,A_i^\star}(\pi)
=
\pi_i(A_i^\star)
\bigl(Q_i^\pi(A_i^\star)-J(\pi)\bigr)
\ge
c\delta q_i
\ge
\frac{c\delta}{N}D(\pi).
\end{equation}
The claim follows by taking the Euclidean norm.

It remains to bound the baseline-free noise ratio on $\Pi_{\rm all}(c)$. For any joint action $\ba\neq\ba^\star$, change each non-optimal coordinate of $\ba$ to its optimal arm one at a time. Each such change increases the reward by at least $\delta$. Hence, $J^\star-r(\ba)\ge\delta$ whenever $\ba\neq\ba^\star$. Therefore, for any product policy in $\Pi_{\rm all}(c)$,
\begin{equation}
D(\pi)=J^\star-J(\pi)
\ge
\delta\Pr_{\ba\sim\pi}\{\ba\neq\ba^\star\}.
\end{equation}
Since $\pi_i(x)\ge c$ for every local arm,
\begin{equation}
\begin{aligned}
\Pr_{\ba\sim\pi}\{\ba\neq\ba^\star\}
&=
1-\prod_i\pi_i(A_i^\star)\\
&\ge
1-\prod_i\bigl(1-(K_i-1)c\bigr)
=\mu(c).
\end{aligned}
\end{equation}
Thus, $D(\pi)\ge\delta\mu(c)$ on $\Pi_{\rm all}(c)$. Combining this lower bound with the product-gradient dominance bound gives
\begin{equation}
\|\mathcal G(\pi)\|_2^2
\ge
\left(\frac{c\delta}{N}\right)^2D(\pi)^2
\ge
\frac{c^2\delta^4\mu(c)^2}{N^2}.
\end{equation}
Finally, because $R_t\in[0,1]$ and
$\|\mathbf e_{A_i}-\pi_i\|_2^2\le2$,
\begin{equation}
\mathcal V(\pi)
\le
2N .
\end{equation}
The displayed bound on $M_{\rm all}(c)$ follows.
\end{proof}

\section{Supplementary Theoretical Results}
\label{app:supp-results}

This section presents additional theoretical results that strengthen or extend the main-text analysis.

\subsection{Constant-Probability Stochastic Failure for the Block Trap}
\label{app:stochastic-trap}

Main-Text Proposition 2 establishes a deterministic product-gradient basin for the block-trap table. The next result provides a stochastic counterpart under the same zero-preference initialization: with a standard diminishing stepsize schedule and a sufficiently small initial scale, the failure probability of the baseline-free stochastic process is bounded below by a positive constant that is independent of the horizon. This result complements the finite-horizon constant-stepsize simulations, which illustrate the basin geometry but do not constitute an almost-sure convergence theorem.

\begin{supplemma}[Policy-space lock-in]
\label{lem:policy-lockin}
Let $(\F_t)$ be a filtration and let $X_t$ be an $\F_t$-adapted stochastic approximation on a compact set $\mathcal X\subset\R^m$:
\begin{equation}
X_{t+1}
=
X_t+\gamma_t\bigl(F(X_t)+U_{t+1}\bigr)+\xi_{t+1},
\end{equation}
where $F$ is Lipschitz, $U_{t+1}$ is $\F_{t+1}$-measurable with $\E[U_{t+1}\mid\F_t]=0$ and $\|U_{t+1}\|_2\le B$, and the $\F_{t+1}$-measurable remainder satisfies $\|\xi_{t+1}\|_2\le B_r\gamma_t^2$. Assume that $\{\gamma_t\}$ is deterministic and nonnegative, $\sum_t\gamma_t=\infty$, and $\sum_t\gamma_t^2<\infty$. If $\mathcal T$ is a compact asymptotically stable attractor of $\dot X=F(X)$ and $K$ is a compact subset of its basin, then constants $a_K,b_K>0$ and $C_K<\infty$ exist such that
\begin{equation}
\Pr\{\operatorname{dist}(X_t,\mathcal T)\to0\}
\ge
\left[1-C_K\sum_{t=1}^{\infty}\gamma_t^2\right]_+
\end{equation}
where $[z]_+=\max\{z,0\}$. This bound holds whenever $X_1\in K$, $\max_t\gamma_t\le a_K$, and $\sum_t\gamma_t^2\le b_K$.
\end{supplemma}

\begin{proof}
Let $\Phi_s$ denote the flow of $\dot X=F(X)$. Asymptotic stability provides nested neighborhoods
\begin{equation}
\mathcal T\subset U_0,
\qquad
\overline U_0\subset U_1,
\qquad
\overline U_1\subset U_2,
\qquad
\overline U_2\subset\mathcal B(\mathcal T),
\end{equation}
where $U_2$ can be chosen as a positively invariant attracting neighborhood of $\mathcal T$. Compactness of $K$ and attraction imply that there is a finite $s_K$ such that $\Phi_{s_K}(K)\subset U_0$. The flow tube
\begin{equation}
\mathcal C_K=\{\Phi_s(x):x\in K,\ 0\le s\le s_K\}
\end{equation}
is compact and lies inside the basin. Hence, the tube and the nested neighborhood boundaries have a positive robustness margin $\varepsilon_K$.

Consider the affine interpolation of the recursion in ODE time. On every interval of bounded ODE length, the Lipschitz property and Gronwall's inequality bound its distance from the corresponding flow segment by a constant times the oscillation of the cumulative perturbation on that interval, plus the Euler and remainder errors. The total nonmartingale error is at most
\begin{equation}
C_F\sum_t\gamma_t^2+B_r\sum_t\gamma_t^2.
\end{equation}
After reducing $a_K$ and $b_K$, this quantity is below $\varepsilon_K/4$. Thus, if every martingale increment over a flow-time window is below $\varepsilon_K/2$, the interpolation follows $\mathcal C_K$ into $U_1$ and remains in $U_2$.

Let
\begin{equation}
\mathsf M_n=\sum_{t=1}^{n}\gamma_tU_{t+1}.
\end{equation}
The martingale-difference property, boundedness, and $\sum_t\gamma_t^2<\infty$ imply that $\mathsf M_n$ converges almost surely and in $L^2$ to a limit $\mathsf M_\infty$, with
\begin{equation}
\E\|\mathsf M_\infty\|_2^2
\le
B^2\sum_{t=1}^{\infty}\gamma_t^2.
\end{equation}
Doob's maximal inequality therefore yields
\begin{equation}
\Pr\!\left\{\sup_{n\ge1}\|\mathsf M_n\|_2>\frac{\varepsilon_K}{4}\right\}
\le
\frac{64B^2}{\varepsilon_K^2}
\sum_{t=1}^{\infty}\gamma_t^2.
\end{equation}
On the complementary event, every martingale increment satisfies
\begin{equation}
\|\mathsf M_n-\mathsf M_m\|_2
\le \frac{\varepsilon_K}{2},
\qquad m,n\ge0,
\end{equation}
so the preceding robustness argument applies to every flow-time window. Moreover, the martingale tail oscillations vanish because $\mathsf M_n$ converges, and the remainder and Euler errors are summable. The interpolation is therefore an asymptotic pseudotrajectory that enters and remains in $U_2$. The standard attractor result for asymptotic pseudotrajectories then gives $\operatorname{dist}(X_t,\mathcal T)\to0$ \cite{benaim1999dynamics}. Absorbing the displayed maximal-inequality coefficient into $C_K$ proves the claim.
\end{proof}

\begin{supproposition}[Constant-probability stochastic trapping]
\label{prop:stochastic-trap}
Consider the baseline-free independent gradient-bandit update on the block-trap table from Main-Text Proposition 2, initialized with zero preferences and using deterministic rewards. Let
\begin{equation}
\gamma_t=\frac{\alpha}{t+t_0},
\qquad \alpha>0,
\end{equation}
where $t_0$ is sufficiently large. Then, there exists a constant $p_{\rm trap}>0$, independent of the horizon, such that, with probability at least $p_{\rm trap}$,
\begin{equation}
q_{i,t}=\pi_{i,t}(2)<\frac{9}{19}
\quad\text{for all }t,
\qquad
q_{i,t}\to0
\quad\text{for }i=1,2 .
\end{equation}
On this event, the optimal-profile probability $\pi_{1,t}(2)\pi_{2,t}(2)$ converges to $0$, the objective converges to the suboptimal block value $0.9$, and the product-policy gap converges to $0.1$. Denote this gap by $D_t=1-J_{\rm trap}(q_{1,t},q_{2,t})$. More precisely, there exists a finite constant $C$, depending only on the table and the zero-preference initialization, such that, for every fixed $\alpha>0$ and all sufficiently large $t_0$,
\begin{equation}
p_{\rm trap}
\ge
\left[1-\frac{C\alpha^2}{t_0}\right]_+ .
\end{equation}
In particular, increasing $t_0$ if necessary gives $p_{\rm trap}\ge1/2$.
The default initialization also fails under the natural schedule $\gamma_t=\alpha/(t+1)$: there exist constants $\alpha_0>0$ and $C_0<\infty$, depending only on the block-trap table and chosen such that $C_0\alpha_0^2\le1/2$, for which every $\alpha\in(0,\alpha_0]$ satisfies
\begin{equation}
\Pr\{q_{i,t}\to0\text{ for }i=1,2\}
\ge
\left[1-C_0\alpha^2\right]_+ .
\end{equation}
Consequently,
\begin{equation}
\liminf_{t\to\infty}\E[D_t]
\ge
0.1\left[1-C_0\alpha^2\right]_+ .
\end{equation}
\end{supproposition}

\begin{proof}
Let $X_t=(\pi_{1,t},\pi_{2,t})$ denote the product policy. The exact preference update induces the following multiplicative policy recursion:
\begin{equation}
\pi_{i,t+1}(x)
=
\frac{\pi_{i,t}(x)\exp(\gamma_tg_{i,t}(x))}
{\sum_y\pi_{i,t}(y)\exp(\gamma_tg_{i,t}(y))}.
\end{equation}
Because $|g_{i,t}(x)|\le1$, a uniform Taylor expansion on the compact product simplex yields
\begin{equation}
X_{t+1}
=
X_t+\gamma_t\bigl(F(X_t)+U_{t+1}\bigr)+\xi_{t+1},
\end{equation}
where $U_{t+1}$ is a bounded martingale difference, $\|\xi_{t+1}\|_2\le B_r\gamma_t^2$ for a table-dependent constant $B_r$, and the deterministic policy flow induced by $\dot{\HH}=\nabla J(\HH)$ has coordinates
\begin{equation}
F_{i,x}(\pi)
=
\pi_i(x)\left(
\mathcal G_{i,x}(\pi)
-\sum_y\pi_i(y)\mathcal G_{i,y}(\pi)
\right).
\end{equation}
The field $F$ is polynomial in the policy coordinates and therefore Lipschitz on the compact product simplex.

Let
\begin{equation}
\mathcal T
=
\{(\pi_1,\pi_2):\pi_1(2)=\pi_2(2)=0\}.
\end{equation}
The exact reduced policy flow derived above is
\begin{equation}
\begin{aligned}
\dot q_i
&=
q_i^2
\left((1-q_i)^2+\pi_i(0)^2+\pi_i(1)^2\right)\\
&\quad{}\times
\left(1.9q_j-0.9\right),
\qquad j\neq i.
\end{aligned}
\end{equation}
Consequently, $\mathcal T$ is a compact asymptotically stable attractor, and the open set $q_1,q_2<9/19$ is contained in its basin. In particular, the singleton consisting of the zero-preference policy $q_1=q_2=1/3$ is a compact subset of this basin.

For $\gamma_t=\alpha/(t+t_0)$,
\begin{equation}
\max_t\gamma_t\le\frac{\alpha}{t_0},
\qquad
\sum_{t=1}^{\infty}\gamma_t^2
\le
\frac{\alpha^2}{t_0}.
\end{equation}
Supplementary Lemma~\ref{lem:policy-lockin} therefore gives
\begin{equation}
\Pr\{q_{1,t},q_{2,t}\to0\}
\ge
\left[1-\frac{C\alpha^2}{t_0}\right]_+
\end{equation}
for all sufficiently large $t_0$. On the robust event used by the lemma, the interpolation remains in $q_1,q_2<9/19$, so these inequalities hold at every iterate. For the natural schedule $\gamma_t=\alpha/(t+1)$, we instead have $\max_t\gamma_t\le\alpha/2$ and $\sum_t\gamma_t^2\le\alpha^2$. The same lemma yields constants $\alpha_0>0$ and $C_0<\infty$ such that
\begin{equation}
\Pr\{q_{1,t},q_{2,t}\to0\}
\ge
\left[1-C_0\alpha^2\right]_+
\ge\frac12
\end{equation}
for every $\alpha\in(0,\alpha_0]$, after reducing $\alpha_0$ if necessary. On this event, $D_t\to0.1$. Since $D_t\ge0$, Fatou's lemma yields
\begin{equation}
\liminf_{t\to\infty}\E[D_t]
\ge
\E[\liminf_{t\to\infty}D_t]
\ge
0.1\left[1-C_0\alpha^2\right]_+ .
\end{equation}
\end{proof}

\paragraph{Finite-horizon witness.}
The preceding proposition concerns asymptotic behavior and uses diminishing stepsizes. A simpler explicit finite-horizon event is also available for any fixed horizon and constant stepsize. Fix $T\ge1$ and $\eta>0$, and run the baseline-free independent gradient-bandit update on the same block-trap table from zero preferences.
Let $E_T$ denote the event that both players sample arm $0$ in every round $t\le T$. This event lies within the suboptimal block. Conditional on $E_{t-1}$, symmetry ensures that the preferences of arms $1$ and $2$ remain equal for each player, while the preference of arm $0$ is no smaller than either of them. Hence, on each prefix of the event, every player samples arm $0$ with probability at least $1/3$ in every round. Therefore,
\begin{equation}
\Pr(E_T)\ge (1/9)^T .
\end{equation}

On $E_T$, the equality between the preferences of arms $1$ and $2$ implies
$\pi_{i,t}(1)=\pi_{i,t}(2)$ for every $i$ and $t\le T+1$. Because arm $0$ is sampled repeatedly and receives the positive block reward, its preference never falls below the common preference of arms $1$ and $2$. Thus, $\pi_{i,T+1}(2)\le1/3$. The optimal-profile probability is the product of the two optimal-arm probabilities and is therefore at most $1/9$. Finally, for any $q_1,q_2\le1/3$,
\begin{equation}
J_{\rm trap}(q_1,q_2)
=0.9(1-q_1)(1-q_2)+q_1q_2
\le0.9,
\end{equation}
and hence $J^\star-J_{\rm trap}(q_1,q_2)\ge0.1$.

\subsection{Strict Trap Structure Beyond One Matrix}

The following results establish the trap-minor obstruction in Main-Text Proposition 3 beyond a single numerical table. The classification characterizes strict traps, the embedding result extends them to larger tables, and the lock-in result transfers deterministic attraction to the sampled update. These statements are intentionally restricted to strict traps: degenerate ties can create neutral continua, whereas the separate block-trap result above addresses the tied example under default initialization.

\begin{supproposition}[Complete strict-trap classification for $2\times2$ tables]
\label{prop:two-by-two-classification}
Consider any deterministic two-player, two-arm common-reward table
\begin{equation}
\begin{array}{c|cc}
 & 0 & 1 \\ \hline
0 & a & b\\
1 & c & d
\end{array}.
\end{equation}
Under deterministic product-gradient flow, a strict, attracting pure-product trap with a suboptimal value exists if and only if the table contains a suboptimal strict coordinate-wise local maximum. Equivalently, at least one of the following strict conditions holds:
\begin{equation}
\begin{array}{ll}
a>\max\{b,c\}\ \text{and}\ a<\max\{a,b,c,d\},\\
b>\max\{a,d\}\ \text{and}\ b<\max\{a,b,c,d\},\\
c>\max\{a,d\}\ \text{and}\ c<\max\{a,b,c,d\},\\
d>\max\{b,c\}\ \text{and}\ d<\max\{a,b,c,d\}.
\end{array}
\end{equation}
\end{supproposition}

\begin{proof}
Let $p=\pi_1(0)$ and $q=\pi_2(0)$. The product objective is
\begin{equation}
J(p,q)=pqa+p(1-q)b+(1-p)qc+(1-p)(1-q)d .
\end{equation}
The sign of $\dot p$ equals the sign of the local advantage of arm $0$ over arm $1$ for player 1,
\begin{equation}
F_1(q)=q(a-c)+(1-q)(b-d),
\end{equation}
and the sign of $\dot q$ equals the sign of the local advantage of arm $0$ over arm $1$ for player 2,
\begin{equation}
F_2(p)=p(a-b)+(1-p)(c-d).
\end{equation}
Thus, the corner $(p,q)=(1,1)$, which corresponds to payoff $a$, is strictly attracting if and only if $F_1(1)>0$ and $F_2(1)>0$, namely, if $a>c$ and $a>b$. It is a spurious trap exactly when $a$ is also below the global table maximum. Applying the same inward-sign calculation at the other three corners yields the four displayed alternatives.

It remains to rule out other strictly attracting product-policy equilibria. If $a-b-c+d\neq0$, any interior stationary point is the intersection of the two affine indifference lines. The off-diagonal entries in the linearization of the $(p,q)$ flow have a strictly positive product, so the point is a saddle. If $a-b-c+d=0$, the indifference conditions are constant or define non-isolated neutral sets. Hence, no additional strictly attracting interior trap exists. A boundary equilibrium that is not a strict corner requires a tie in the local advantage of one player and is therefore not strictly attracting. This proves the equivalence.
\end{proof}

\begin{supproposition}[Trap-minor embedding for finite-arm tables]
\label{prop:trap-minor}
Consider any deterministic two-player common-reward table with finite arm sets. Suppose that there are arms $x_0,x_1\in\A_1$ and $y_0,y_1\in\A_2$ such that the $2\times2$ minor on these arms has a strict suboptimal trap at $(x_0,y_0)$:
\begin{equation}
\begin{aligned}
r(x_0,y_0)&>r(x_1,y_0),\\
r(x_0,y_0)&>r(x_0,y_1),\\
r(x_0,y_0)&<J^\star .
\end{aligned}
\end{equation}
Assume further that the trap corner is strict with respect to all outside unilateral deviations,
\begin{equation}
\begin{aligned}
r(x_0,y_0)&>r(x,y_0)\quad \forall x\neq x_0,\\
r(x_0,y_0)&>r(x_0,y)\quad \forall y\neq y_0 .
\end{aligned}
\end{equation}
Then, $(x_0,y_0)$ is a strict, attracting pure-product trap for the full product-gradient flow. In particular, there exists a neighborhood of product policies concentrated on $(x_0,y_0)$ whose deterministic flow moves toward that suboptimal corner.
\end{supproposition}

\begin{proof}
For player 1, when player 2 assigns probability close to one to $y_0$, every alternative arm $x\neq x_0$ has a conditional value strictly below that of $x_0$. This follows because the displayed inequalities are strict and the conditional values vary continuously with the opponent policy. Thus, there exists a neighborhood of $y_0$ in which all preference differences $H_1(x_0)-H_1(x)$ have positive drift. The same argument applies to player 2 around $x_0$. Therefore, a product-policy neighborhood of the corner has inward drift for every unilateral deviation away from $(x_0,y_0)$. Since $r(x_0,y_0)<J^\star$, the attracting corner is suboptimal. The $2\times2$ minor conditions identify the local coordination-trap pattern, while the outside-deviation inequalities ensure that the minor remains attracting in the full table.
\end{proof}

\begin{supproposition}[Stochastic lock-in for strict trap minors]
\label{prop:trap-minor-stochastic}
Consider a deterministic finite two-player common-reward table satisfying the trap-minor conditions of Supplementary Proposition~\ref{prop:trap-minor}. Let $\mathcal B$ denote the basin of the resulting suboptimal pure-product attractor. Fix $\alpha>0$, and run the baseline-free independent gradient-bandit update with stepsizes
\begin{equation}
\gamma_t=\frac{\alpha}{t+t_0}.
\end{equation}
For every compact set $K\subset\mathcal B$, there exist constants $C_K<\infty$ and $t_K<\infty$, where $t_K$ may depend on $\alpha$, such that the following statement holds: for every finite-preference initialization whose induced product policy lies in $K$ and every $t_0\ge t_K$, the stochastic product policies converge to the suboptimal trap with probability at least
\begin{equation}
\left[1-\frac{C_K\alpha^2}{t_0}\right]_+ .
\end{equation}
Let
\begin{equation}
\Delta_{\rm trap}=J^\star-r(x_0,y_0)>0.
\end{equation}
Then the expected product-policy gap also satisfies
\begin{equation}
\liminf_{t\to\infty}\E[J^\star-J(\pi_t)]
\ge
\Delta_{\rm trap}
\left[1-\frac{C_K\alpha^2}{t_0}\right]_+ .
\end{equation}
\end{supproposition}

\begin{proof}
The strict unilateral inequalities in Supplementary Proposition~\ref{prop:trap-minor} make the trap corner an asymptotically stable attractor of the deterministic policy flow. The multiplicative softmax expansion used in the proof of Supplementary Proposition~\ref{prop:stochastic-trap} expresses the stochastic product-policy recursion in the form required by Supplementary Lemma~\ref{lem:policy-lockin}, with a bounded martingale difference and a uniform $O(\gamma_t^2)$ remainder. Finally,
\begin{equation}
\max_t\gamma_t\le\frac{\alpha}{t_0},
\qquad
\sum_{t=1}^{\infty}\gamma_t^2\le\frac{\alpha^2}{t_0}.
\end{equation}
Applying the lemma to $K\subset\mathcal B$ proves the stated bound after increasing $t_K$, if necessary.
On the lock-in event, $J(\pi_t)\to r(x_0,y_0)$, whereas $J^\star-J(\pi_t)\ge0$ on every sample path. Fatou's lemma therefore gives the displayed expected-gap lower bound.
\end{proof}

\begin{supproposition}[A block-trap family]
\label{prop:block-trap-family}
Let each player $i\in\{1,2\}$ have one optimal arm $o_i$ and a nonempty block $B_i$ of $m_i$ non-optimal arms. Consider the deterministic rewards
\begin{equation}
\begin{aligned}
r(o_1,o_2)&=1,\\
r(b_1,b_2)&=\rho,\\
r(o_1,b_2)&=r(b_1,o_2)=0
\end{aligned}
\end{equation}
for all $b_i\in B_i$, where $\rho\in(0,1)$. If
\begin{equation}
\frac{1}{m_i+1}<\frac{\rho}{1+\rho}
\quad\text{for }i=1,2,
\end{equation}
then the default uniform initialization lies in an invariant bad basin of the deterministic product-gradient flow. Within this basin, the flow moves both optimal-arm probabilities toward zero. Moreover, under the same diminishing-stepsize conditions as Supplementary Proposition~\ref{prop:stochastic-trap}, the baseline-free independent gradient-bandit dynamics have a horizon-independent constant lower bound on the probability of convergence to the suboptimal block $B_1\times B_2$.
\end{supproposition}

\begin{proof}
Let $q_i=\pi_i(o_i)$. For every table in this family, define the reduced objective by
\begin{equation}
J_\rho(q_1,q_2):=J(\HH)
=\rho(1-q_1)(1-q_2)+q_1q_2 .
\end{equation}
The conditional value of $o_1$ is $q_2$, whereas every arm in $B_1$ has conditional value $\rho(1-q_2)$. Hence, the direction that increases the probability of the optimal arm of player 1 relative to the block has the sign of
\begin{equation}
q_2-\rho(1-q_2)=(1+\rho)q_2-\rho .
\end{equation}
The expression for player 2 is symmetric. Therefore, the separator is
\begin{equation}
q_1=q_2=\frac{\rho}{1+\rho}.
\end{equation}
The displayed condition states precisely that the uniform initialization $q_i=1/(m_i+1)$ lies below this separator for both players. Below the separator, both optimal-arm directions are negative; thus, the lower-left region is invariant and attracts the flow toward $q_1=q_2=0$. The stochastic statement follows by applying Supplementary Lemma~\ref{lem:policy-lockin} to this attracting block face.
\end{proof}

\paragraph{Where single-player and game-learning intuitions fail.}
The analysis of a single-player stochastic gradient bandit relies on the fact that the advantage of an optimal arm is determined by that arm relative to the current scalar objective. In particular, the conditional value of the optimal arm cannot deteriorate merely because another learner has not yet coordinated with it. In an \CRMPMAB{}, however, player $i$ compares $Q_i^\pi(A_i^\star)$ with $J(\pi)$, and $Q_i^\pi(A_i^\star)$ averages over the current product policy of the other players. In the block trap,
\begin{equation}
Q_1^\pi(2)-Q_1^\pi(0)=q_2-0.9(1-q_2)=1.9q_2-0.9 .
\end{equation}
At uniform initialization, this quantity is negative. Thus, the same softmax gradient mechanism that drives a single player toward a globally best arm can drive an oblivious player away from its component of the globally optimal joint action when the other player has not yet assigned sufficient probability to the matching component. This is precisely the product-policy obstruction ruled out by the certificate $\lambda_\star(c)>0$.

This setting also differs from standard common-interest game dynamics. A common-interest game typically provides each learner with a payoff function indexed by a known action profile or analyzes a specified payoff-based revision protocol over a known game form. In an \CRMPMAB{}, a player observes neither the action profile nor the game form; the player observes only a scalar sample after choosing a local arm. The product-policy objective and the quantities $Q_i^\pi(x)$ are therefore analytical objects rather than data available to the learner. The framework in the main text determines which conclusions survive this loss of information: unbiased stationarity survives, global optimality does not, and certificates identify reward tables for which the missing information is not harmful.

\subsection{Dense Additive-Interaction Extension}
\label{app:local-dominance}

\begin{suppcorollary}[Additive rewards with bounded interactions]
\label{cor:additive-perturbation}
Suppose that each player has an optimal local arm $A_i^\star$, and
\begin{equation}
r(\ba)=\frac{1}{N}\sum_{i=1}^{N}u_i(A_i)+\epsilon h(\ba),
\qquad
0\le r(\ba)\le1,
\end{equation}
where $u_i(A_i^\star)-u_i(x)\ge\Delta$ for every $i$ and $x\neq A_i^\star$, and
$|h(\ba)-h(\ba')|\le B$ whenever $\ba$ and $\ba'$ differ in one coordinate. If
\begin{equation}
\delta=\frac{\Delta}{N}-|\epsilon|B>0,
\end{equation}
then the local-dominance condition holds with this value of $\delta$. Consequently,
\begin{equation}
\lambda_\star(c)\ge\left(\frac{c}{N}\left(\frac{\Delta}{N}-|\epsilon|B\right)\right)^2 .
\end{equation}
Moreover, the explicit bound on $M_{\rm all}(c)$ in Main-Text Proposition 4 also applies with this value of $\delta$.
\end{suppcorollary}

\begin{proof}
For any player $i$, any $x\neq A_i^\star$, and any fixed $\ba_{-i}$,
\begin{equation}
\begin{aligned}
r(A_i^\star,\ba_{-i})-r(x,\ba_{-i})
&=
\frac{u_i(A_i^\star)-u_i(x)}{N}\\
&\quad+
\epsilon\bigl(h(A_i^\star,\ba_{-i})-h(x,\ba_{-i})\bigr)\\
&\ge
\frac{\Delta}{N}-|\epsilon|B
=\delta .
\end{aligned}
\end{equation}
The result then follows from Main-Text Proposition 4.
\end{proof}

\section{Experiments}
\label{app:experiments}

\subsection{Experiment Setup}

The experiments examine the product-policy geometry identified by the theory. Unless otherwise stated, all experiments use the same baseline-free independent gradient-bandit update analyzed in the paper. The primary trajectories use $200$ runs and $5000$ steps, with $\eta=0.20$ for sparse coordination and the block trap and $\eta=0.15$ for the additive and AI-pipeline instances. When all arm sets have the same cardinality, we write $K=K_i$ for every player. The trajectory plots report means over independent runs with 95\% confidence bands computed as the mean $\pm1.96$ standard errors.

\paragraph{Implementation environment.}
The archived results were regenerated without a GPU on Ubuntu 24.04.4 LTS using two Intel Xeon Gold 5318Y CPUs, 251 GiB of memory, Python 3.13.9, and Matplotlib 3.10.6. The simulation code otherwise uses only the Python standard library, and the complete default suite takes approximately eight minutes on this machine.

\paragraph{Certificate computation.}
For a finite reward table and a chosen optimal joint action $\ba^\star$, the structural certificates are offline properties of the product-policy landscape. First, choose a support threshold $c$. The deterministic no-trap ratio is evaluated over $\Pi_\star(c)$, whereas the conservative baseline-free noise ratio is evaluated over $\Pi_{\rm all}(c)$, unless a sharper invariant region is known. For each product policy, compute $J(\pi)$, the local conditional values $Q_i^\pi(x)$, the product-policy gradient $\mathcal G(\pi)$, the gap $D(\pi)=J^\star-J(\pi)$, and the second-moment quantity $\mathcal V(\pi)$. The certificate values are
\begin{equation}
\inf_{D(\pi)>0}\frac{\|\mathcal G(\pi)\|_2^2}{D(\pi)^2},
\qquad
\sup_{\|\mathcal G(\pi)\|_2>0}
\frac{\mathcal V(\pi)}{\|\mathcal G(\pi)\|_2^2}
\end{equation}
over the selected region.
As in the main text, a finite noise certificate also requires $\mathcal V(\pi)=0$ at every zero-gradient policy in the selected region; its value is set to $0$ when no positive-gradient policy remains and $\mathcal V$ vanishes throughout the region.

We implement a fixed-depth checker for small binary-arm reward tables. After relabeling the local arms so that $\ba^\star=(0,\ldots,0)$, we parameterize each product policy by $p_i=\pi_i(0)$, partition $[0,1]^N$ into dyadic boxes $B$ at a prescribed depth, and intersect these boxes with the selected certificate region. On each box, rational interval arithmetic provides the certified enclosures
\begin{equation}
\begin{aligned}
\underline{D}(B)&\le D(\pi)\le\overline{D}(B),\\
\underline{G^2}(B)&\le\|\mathcal G(\pi)\|_2^2,\\
\mathcal V(\pi)&\le\overline{V}(B).
\end{aligned}
\end{equation}
For any box with $\overline{D}(B)>0$, the value
$\underline{G^2}(B)/\overline{D}(B)^2$ is a valid lower bound on the deterministic ratio. The implementation combines these interval values with a separately established analytic lower bound $\lambda_{\rm an}$; it does not infer the positivity of $\lambda_\star(c)$ from boxes that intersect the zero-gap set. For the noise ratio, every box that intersects $\Pi_{\rm all}(c)$ must satisfy $\underline{G^2}(B)>0$, in which case $\overline{V}(B)/\underline{G^2}(B)$ is a valid upper bound. If any such box intersects the zero-gradient set, the checker returns \textsc{unknown}. The current implementation neither adaptively subdivides unresolved boxes nor returns a trap witness.

The checker is executable because $J$, $Q_i^\pi(x)$, $\mathcal G$, and $\mathcal V$ are finite sums and products of rational reward entries and policy coordinates. It is therefore a hybrid analytic/interval procedure: the positivity of the deterministic certificate is established analytically, while fixed-depth rational intervals independently certify the all-arm noise bound. Every reported pair $(\underline{\lambda},\overline M)$ is sound; cases with insufficient interval resolution are reported as \textsc{unknown}. Analytic certificates remain preferable when available: Main-Text Corollary 1 certifies sparse coordination, and Main-Text Proposition 4 provides both certificates for every locally dominant dense table. The AI-pipeline routing instance below is dense and nonadditive but locally dominant with $\delta=0.11$. Grid and random-policy evaluations remain diagnostics rather than certificates.

Table~\ref{tab:verified-certificates} reports the results of the fixed-depth checker at depth four. Its dense interaction row uses the normalized nonadditive table
\begin{equation}
\begin{aligned}
r(\ba)
&=\frac{1}{10}\biggl(2+2\sum_{i=1}^3\one\{A_i=0\}\\
&\qquad+\one\{A_1=A_2=0\}\\
&\qquad+\one\{A_2=A_3=0\}\biggr),
\end{aligned}
\end{equation}
which is locally dominant with $\delta=0.2$. The displayed values are conservative certified bounds.

The horizon-dependent constants in Main-Text Theorem 2 should be applied through an offline accept/reject workflow rather than used as an adaptive online rule. First, choose a horizon $T$ and a candidate stepsize schedule, compute $c_T=e^{-2\Gamma_T}/K_{\max}$ or a sharper invariant region, and then obtain certified constants $(\lambda,M)$ on that region analytically or with the hybrid checker. The final step is to verify that $\gamma_t\le 1/(LM)$ for every $t\le T$. If this condition fails, reduce the stepsize or certify a sharper invariant region and repeat the procedure. Thus, the condition is circular only if the certificate is treated as unknown; under an offline check, it is a finite sufficient-condition test for a proposed run.

\begin{table}[tb]
\centering
\small
\begin{tabular}{@{}lccc@{}}
\toprule
Instance & Region $c$ & $\underline{\lambda}$ & $\overline M$ \\
\midrule
\shortstack[l]{Sparse\\$N=3,K=2$} & $0.25$ & $2.71{\times}10^{-5}$ & $148.76$ \\
\shortstack[l]{Dense additive\\$N=3,K=2$} & $0.25$ & $6.94{\times}10^{-3}$ & $147.70$ \\
\shortstack[l]{Dense interaction\\$N=3,K=2$} & $0.25$ & $2.78{\times}10^{-4}$ & $346.23$ \\
\bottomrule
\end{tabular}
\caption{Hybrid analytic and interval certificates for binary-arm tables. The checker combines analytic lower bounds for $\lambda_\star(c)$ with fixed-depth rational interval upper bounds for $M_{\rm all}(c)$.}
\label{tab:verified-certificates}
\end{table}

All three rows use the finite-horizon regions $\Pi_\star(0.25)$ and $\Pi_{\rm all}(0.25)$. For sparse coordination, the first value is the analytic bound $c^{2N}/N^2$ at $c=0.25$, and the second is the fixed-depth interval bound on $M_{\rm all}(0.25)$. These values differ from the invariant-region constants in Main-Text Corollary 1: for $N=3$, $K=2$, and $J^\star=1$, that corollary instead uses $c_\star=0.5$ and gives $\lambda_\star(c_\star)\ge1/576$ and $M_{\rm sp}\le8$. For the dense additive row, writing $p_i=\pi_i(0)$ gives $D(\pi)=\frac13\sum_i(1-p_i)$; selecting an index with $1-p_i\ge D(\pi)$ yields the instance-specific analytic bound $\lambda_\star(0.25)\ge0.25^2/3^2=1/144$. The dense interaction row uses the local-dominance bound with $\delta=0.2$, which gives $1/3600$. The three values of $\overline M$ are rational interval bounds over the corresponding all-arm region.

\paragraph{Information model and ablations.}
Table~\ref{tab:upper-information-reference} records which rows use the same online information as the independent gradient-bandit update and which require joint actions, communication, or aggregation channels. To separate the product-policy mechanism from variance reduction, we also evaluate an action-independent moving reward baseline initialized at $b_1=0$. Round $t$ uses $R_t-b_t$, followed by $b_{t+1}=b_t+0.05(R_t-b_t)$. Because $b_t$ is fixed before the current actions are sampled, the baseline term has zero conditional mean and changes neither the product-gradient drift nor the deterministic trap geometry. This variant uses the same online information as the baseline-free update: the sampled local arm and the common reward. The stepsize and baseline ablation uses $\eta\in\{0.05,0.10,0.15,0.20,0.30\}$, $80$ runs, and $3000$ steps. The structural pseudo-regret theorem is stated for the baseline-free update because a moving baseline requires a self-bounding certificate for the residual second moment along the joint policy--baseline process.

\begin{table}[tb]
\centering
\small
\begin{tabular}{@{}p{0.24\linewidth}p{0.34\linewidth}p{0.32\linewidth}@{}}
\toprule
Reference & Online information & Interpretation \\
\midrule
Independent update & own arm + common reward & matched learner \\
Moving-baseline variant & own arm + common reward & matched ablation \\
Checker & offline reward table & offline certificate or diagnostic check \\
Centralized learner & joint actions or policy & upper-info reference \\
Comm. bandit & messages or rewards & different protocol \\
\bottomrule
\end{tabular}
\caption{Information-model reference. Centralized or communicating methods can serve as upper-information references, but they are not protocol-matched baselines for fully oblivious learners.}
\label{tab:upper-information-reference}
\end{table}

We also include a small upper-information reference for the block trap. Both rows use $4$ runs, $1500$ steps, and $\eta=0.20$. The centralized row uses a single softmax distribution over the nine joint actions and therefore observes and parameterizes the joint action, which is unavailable to oblivious learners. This row is not a fair protocol-matched baseline; it is included only to indicate the performance cost of the product-policy information restriction.

\begin{table}[tb]
\centering
\small
\begin{tabular}{@{}lccc@{}}
\toprule
Reference & Final gap & $\pi(\ba^\star)$ & \shortstack{Avg.\\pseudo-regret} \\
\midrule
\shortstack[l]{Independent\\default} & $0.1037$ & $4.77{\times}10^{-6}$ & $0.1274$ \\
\shortstack[l]{Centralized\\joint-action} & $0.0270$ & $0.7443$ & $0.0595$ \\
\bottomrule
\end{tabular}
\caption{Upper-information reference on the block trap. The centralized learner has access to joint-action information and is not protocol matched to the independent update.}
\label{tab:upper-information-results}
\end{table}

\paragraph{Instances.}
The sparse coordination instance uses the reward class and invariant support geometry of Main-Text Corollary 1. The optimal joint action is $\ba^\star=(0,\ldots,0)$, each player has two arms, and the deterministic reward is
\begin{equation}
  r(\ba)=\one\{\ba=\ba^\star\}.
\end{equation}
We vary the number of players over $N\in\{3,5,8\}$.
The practical choice $\eta=0.20$ exceeds the corollary's conservative sufficient limits $1/96$, $1/640$, and $1/8192$, respectively. These trajectories are therefore mechanism diagnostics and do not empirically verify the certified rate constants.

The dense additive instance uses $N=6$ players and $K_i=4$ arms per player. Arm $0$ is optimal for every player, and the deterministic mean reward is
\begin{equation}
  r(\ba)=\frac{1}{N}\sum_{i=1}^N \one\{A_i=0\}.
\end{equation}
The observed reward follows a Bernoulli distribution with this mean.

The AI-pipeline routing instance is a stylized dense \CRMPMAB{} with $N=4$ modules: prompt template, retrieval policy, tool choice, and verification mode. Each module chooses one of three local configurations and observes only the final task-success reward. Arm $0$ is the preferred local configuration. The mean reward is
\begin{equation}
\begin{aligned}
r(\ba)
&=\frac{1}{4}\sum_{i=1}^4 u(A_i)
-0.015\sum_{(i,j)\in E}\one\{A_i\ne A_j\},\\
u&=(1,0.4,0.1),
\end{aligned}
\end{equation}
where $E=\{(1,2),(2,3),(3,4),(1,4)\}$ encodes mild compatibility penalties between adjacent pipeline choices. Although this table is dense and nonadditive, changing any one module from a nonzero arm to arm $0$ improves the reward by at least $0.15-2(0.015)=0.12$. We use the conservative certificate value $\delta=0.11$.

The coordinate block trap is the two-player, three-arm reward matrix used in Main-Text Proposition 2:
\begin{equation}
\begin{array}{c|ccc}
 & A_2=0 & A_2=1 & A_2=2 \\ \hline
A_1=0 & 0.9 & 0.9 & 0\\
A_1=1 & 0.9 & 0.9 & 0\\
A_1=2 & 0 & 0 & 1
\end{array}.
\end{equation}
For the opt-biased trajectory, $H_{i,1}(2)=2$ and the other preferences are zero. For the phase diagram, we scan the two gaps $H_{i,1}(2)-H_{i,1}(0)=H_{i,1}(2)-H_{i,1}(1)$, $i=1,2$, over a $31\times31$ grid on $[-4,4]^2$, using $20$ runs per grid point, $2000$ steps, and $\eta=0.20$.

For the random finite-table stress test, the rewards $r(\ba)$ are sampled independently from the uniform distribution on $[0,1]$, the best joint action is selected as $\ba^\star$, and the ratio diagnostics are evaluated with $c=0.05$ by combining a coarse grid with random product policies. The construction enforces only $\pi_i(A_i^\star)\ge c$, so every sampled policy lies in $\Pi_\star(c)$; it does not enforce the all-arm constraints of $\Pi_{\rm all}(c)$. Consequently, $\widehat M_{\rm rs}(c)$ is a diagnostic ratio over the sampled policies, not an estimate of the formal certificate $M_{\rm all}(c)$.

The plotted $N=3,K=2$ test uses $64$ tables, $16$ runs, $3000$ learning steps, $\eta=0.15$, $8000$ diagnostic policy samples per table, and an $11$-point grid for each optimal-arm probability. The multi-size suite uses $\eta=0.15$, $2000$ learning steps, $1500$ diagnostic policy samples, a $9$-point grid per optimal-arm probability, and $24$ random starts per table. Each initial preference is sampled uniformly from $[-1,1]$, and rewards during learning are Bernoulli samples with the table entry as their mean.

To assess trap robustness, we vary the suboptimal coordinated reward in the family
\begin{equation}
\begin{array}{c|ccc}
 & A_2=0 & A_2=1 & A_2=2 \\ \hline
A_1=0 & \rho & \rho & 0\\
A_1=1 & \rho & \rho & 0\\
A_1=2 & 0 & 0 & 1
\end{array}
\end{equation}
with $\rho\in\{0.60,0.75,0.90,0.95\}$. For each $\rho$, we scan a $21\times21$ grid on $[-4,4]^2$, use $4$ runs per grid point, and run $1500$ steps with $\eta=0.20$.

\FloatBarrier
\subsection{Results}

\begin{figure*}[t]
\centering
\includegraphics[width=0.52\textwidth]{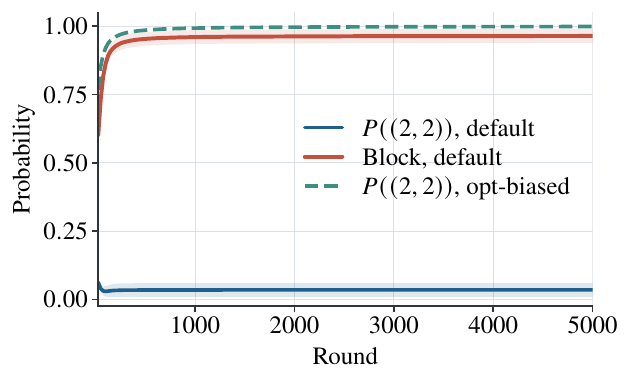}\hfill
\includegraphics[width=0.42\textwidth]{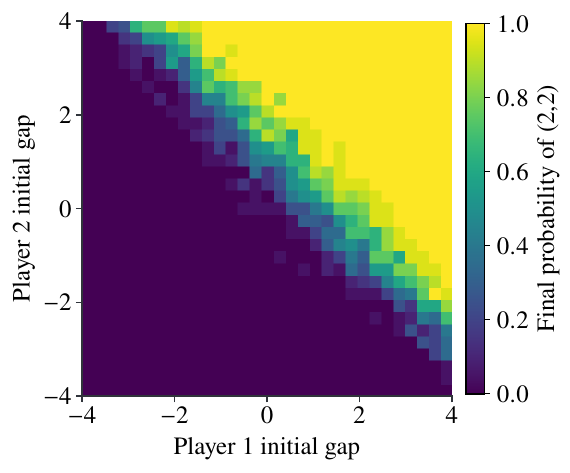}
\caption{Coordinate block trap. Left: trajectories from default and opt-biased initializations. Right: final optimal-profile probability over the initial preference gaps.}
\label{fig:trap-summary}
\end{figure*}

\paragraph{Protocol-matched baseline and stepsize ablations.}
Table~\ref{tab:stepsize-baseline-ablation} reports the best final product-policy gap over the stepsize grid for each variant. The moving baseline leaves the sparse gap unchanged at $0.0011$ and reduces the dense additive gap from $0.0133$ to $0.0054$, consistent with effective variance reduction in a benign landscape. On the block trap, however, the gap changes from $0.0922$ to $0.0963$ and remains bounded away from zero. Thus, the baseline can improve finite-sample optimization without changing the mean product-policy geometry that creates the trap.

\begin{table}[tb]
\centering
\small
\begin{tabular}{@{}lcccc@{}}
\toprule
Instance & \shortstack{No baseline\\$\eta$} & Gap & \shortstack{Moving baseline\\$\eta$} & Gap \\
\midrule
Sparse & $0.30$ & $0.0011$ & $0.30$ & $0.0011$ \\
Additive & $0.20$ & $0.0133$ & $0.30$ & $0.0054$ \\
Block trap & $0.30$ & $0.0922$ & $0.30$ & $0.0963$ \\
\bottomrule
\end{tabular}
\caption{Protocol-matched stepsize and baseline ablation. Both variants use only the sampled local arm and the common reward.}
\label{tab:stepsize-baseline-ablation}
\end{table}

\paragraph{Random finite-table stress test.}
We denote the random-search diagnostics by $\widehat{\lambda}_{\rm rs}(c)$ and $\widehat{M}_{\rm rs}(c)$. The latter is not an estimate of $M_{\rm all}(c)$ because the diagnostic search does not enforce $\Pi_{\rm all}(c)$. Table~\ref{tab:random-suite} reports the multi-size stress suite. Across all three settings, the highest quartile of $\widehat{\lambda}_{\rm rs}(c)$ has a lower mean terminal gap than the lowest quartile. This experiment is an offline stress test of the certificate heuristic rather than a comparison with algorithms that use stronger information.

\noindent\parbox{\linewidth}{%
\centering
\small
\begin{tabular}{@{}lcccc@{}}
\toprule
Setting & Tables & All gap & Low-$\widehat{\lambda}$ gap & High-$\widehat{\lambda}$ gap \\
\midrule
$N=3,K=2$ & $128$ & $0.0413$ & $0.0561$ & $0.0260$ \\
$N=4,K=2$ & $64$ & $0.0561$ & $0.0584$ & $0.0485$ \\
$N=3,K=3$ & $64$ & $0.0701$ & $0.0903$ & $0.0573$ \\
\bottomrule
\end{tabular}
\captionof{table}{Random finite-table stress suite; low and high denote the lowest and highest $\widehat{\lambda}_{\rm rs}(c)$ quartiles.}
\label{tab:random-suite}
}\par

\noindent\parbox{\linewidth}{%
\centering
\includegraphics[width=0.95\linewidth]{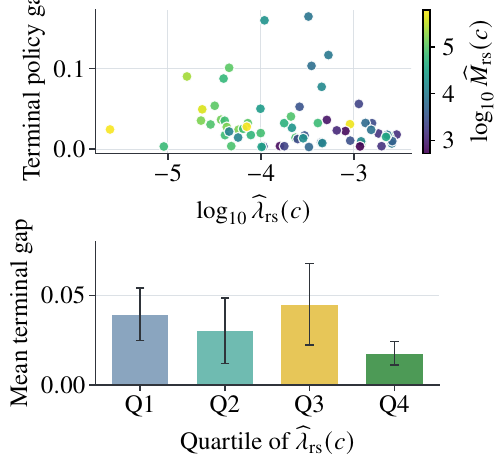}
\captionof{figure}{Random-table stress: terminal gap versus $\widehat{\lambda}_{\rm rs}(c)$ (top; color indicates $\widehat{M}_{\rm rs}(c)$) and quartile means with $\pm1.96$ standard errors (bottom).}
\label{fig:random-table-stress}
}\par

\paragraph{Optimally supported coordination.}
Figure~\ref{fig:opt-supported-summary} shows the product-policy gap and average pseudo-regret. In this instance, the probability assigned to the optimal joint action is exactly one minus the gap. The rewards and invariant support match the structural class of Main-Text Corollary 1, but the practical stepsize exceeds its sufficient range. The plots therefore illustrate the predicted geometry rather than certify the finite-time constants.

\noindent\parbox{\linewidth}{%
\centering
\includegraphics[width=0.95\linewidth]{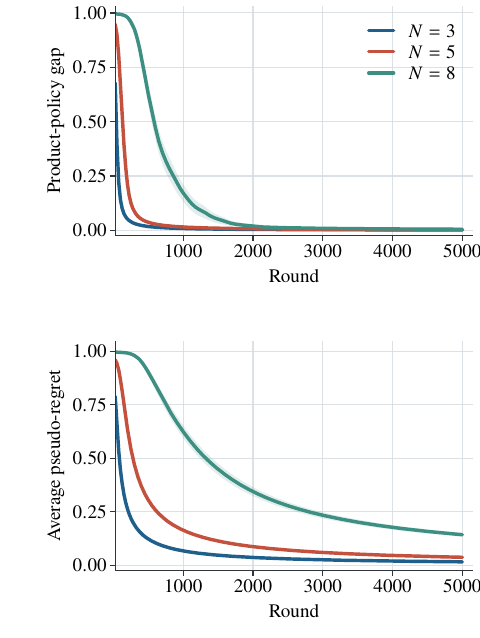}
\captionof{figure}{Optimally supported coordination: product-policy gap (top) and average pseudo-regret (bottom).}
\label{fig:opt-supported-summary}
}\par

\paragraph{Benign additive common reward.}
This dense instance has a benign no-trap geometry: every player has a local move that improves the common reward independently of the other players. The top panel of Figure~\ref{fig:dense-instances} shows the product-policy gap, average pseudo-regret, and probability assigned to the optimal joint action. The experiment illustrates that the structural perspective also applies to dense reward tables beyond the explicit sparse class.

\paragraph{Stylized AI-pipeline routing.}
The AI-pipeline instance is not intended to benchmark a deployed system; instead, it is a compact finite-table case that matches the motivating information pattern. Four pipeline modules independently adapt their local choices using one shared task-success score, and no module observes the choices of the others. The bottom panel of Figure~\ref{fig:dense-instances} shows that the independent update improves both the product-policy gap and the probability assigned to the best pipeline configuration. The certificate for this instance is analytic rather than diagnostic: although the instance is dense and nonadditive, local dominance gives $\lambda_\star(c)\ge(c\delta/4)^2$ and the all-arm noise bound from Main-Text Proposition 4.

\noindent\parbox{\linewidth}{%
\centering
\includegraphics[width=0.95\linewidth]{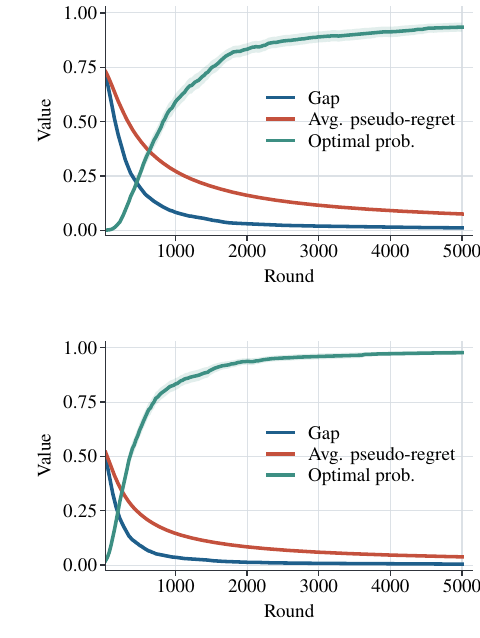}
\captionof{figure}{Dense structured instances: additive reward (top) and AI-pipeline routing (bottom).}
\label{fig:dense-instances}
}\par

\paragraph{Coordinate block trap.}
The globally optimal action is $(2,2)$, but the block $\{0,1\}^2$ is coordinate-wise attractive in the deterministic product-gradient landscape. The left panel of Figure~\ref{fig:trap-summary} compares the default uniform initialization with an opt-biased initialization. Under the default initialization, the finite-horizon stochastic runs remain concentrated on the suboptimal block and assign low probability to $(2,2)$, consistent with the deterministic basin analysis. Supplementary Proposition~\ref{prop:stochastic-trap} complements these constant-stepsize diagnostics by establishing constant-probability trapping for a diminishing-stepsize process; the right panel instead estimates the deterministic basin of attraction. Table~\ref{tab:trap-robustness} shows that this basin grows as the suboptimal coordinated reward approaches the global optimum, supporting the interpretation that the failure mode is an effect of the product-policy landscape rather than a numerical artifact of a single reward table.

\noindent\parbox{\linewidth}{%
\centering
\small
\begin{tabular}{@{}lccc@{}}
\toprule
$\rho$ & Trap basin & Mean $\pi(\ba^\star)$ & Mean objective \\
\midrule
$0.60$ & $0.575$ & $0.423$ & $0.766$ \\
$0.75$ & $0.603$ & $0.396$ & $0.846$ \\
$0.90$ & $0.634$ & $0.364$ & $0.933$ \\
$0.95$ & $0.644$ & $0.354$ & $0.964$ \\
\bottomrule
\end{tabular}
\captionof{table}{Trap robustness versus $\rho$; ``trapped'' means final $\pi(\ba^\star)<0.5$.}
\label{tab:trap-robustness}
}\par

\bibliography{aaai2027}